\documentclass{LeafLabLaTexTemplate}
\usepackage{amsfonts}
\usepackage{amsmath}
\usepackage{amsthm}
\usepackage{amssymb}
\usepackage{mathtools}
\usepackage{algorithm}
\usepackage{algorithmic}
\usepackage{colortbl}
\usepackage{makecell}
\usepackage{multicol}
\usepackage{multirow}
\usepackage{pifont}
\usepackage{xspace}
\usepackage{subfigure}
\usepackage{array}
\usepackage[table]{xcolor}

\usepackage{fontawesome}
\usepackage{academicons}
\usepackage{titletoc}

\usepackage{tabularx}
\usepackage{booktabs}
\usepackage[table]{xcolor}

\usepackage[utf8]{inputenc} % allow utf-8 input
\usepackage[T1]{fontenc}    % use 8-bit T1 fonts
\usepackage{hyperref}       % hyperlinks
\usepackage{bookmark}       % better PDF bookmarks
\usepackage{url}            % simple URL typesetting
\usepackage{booktabs}       % professional-quality tables
\usepackage{amsfonts}       % blackboard math symbols
\usepackage{nicefrac}       % compact symbols for 1/2, etc.
\usepackage{microtype}      % microtypography
\usepackage{xcolor}         % colors
\usepackage{float}
\usepackage{graphicx}
\usepackage{caption}
\usepackage[textsize=tiny]{todonotes}
\newcolumntype{Y}{>{\centering\arraybackslash}X}
\newif\ifaddtitlelogo

\addtitlelogotrue

\title{UBP2: Uncertainty-Balanced Preference Planning}
\newcommand{\alglong}{Uncertainty-Balanced Preference Planning}
\newcommand{\algshort}{UBP2}

\author[]{Mohamed Nabail}
\author[*]{Leo Kaixuan Cheng}
\author[*]{Jingmin Wang}
\author[]{Nicholas Rhinehart}

\affiliation[]{Learning, Embodied Autonomy, and Forecasting (LEAF) Lab, University of Toronto}

\theoremstyle{plain}
\newtheorem{theorem}{Theorem}[section]

\newtheorem{lemma}[theorem]{Lemma}
\newtheorem{corollary}[theorem]{Corollary}
\theoremstyle{definition}
\newtheorem{definition}[theorem]{Definition}
\newtheorem{assumption}[theorem]{Assumption}
\theoremstyle{remark}

\DeclareMathOperator*{\argmax}{arg\,max}

\newcommand{\norm}[1]{\left\lVert #1\right\rVert}

\abstract{
Preference-based RL provides an approach to learning reward models from pairwise comparisons of behaviors, bypassing the need for explicit reward design. However, existing methods typically rely on passive data collection and suffer from poor sample efficiency, especially during the early stages of learning. We introduce a model-based approach that actively directs exploration by jointly reasoning over uncertainties in the reward, dynamics, and value functions. Our method, \alglong~(\algshort), uses ensembles of reward, dynamics, and value function models to evaluate candidate trajectories according to a unified score that combines expected reward, terminal value, and epistemic uncertainty. Planning under this objective yields an explicit tradeoff between exploitation and information acquisition without requiring ad hoc exploration heuristics. Under standard regularity assumptions, we establish sublinear regret guarantees for both finite-horizon and infinite-horizon settings. Empirically, experiments on the Meta-World benchmark show \algshort~achieves substantially higher sample efficiency than model-free preference-based methods and non-optimistic model-based baselines.
}

\metadata[
\raisebox{-0.35em}{\includegraphics[width=0.05\linewidth]{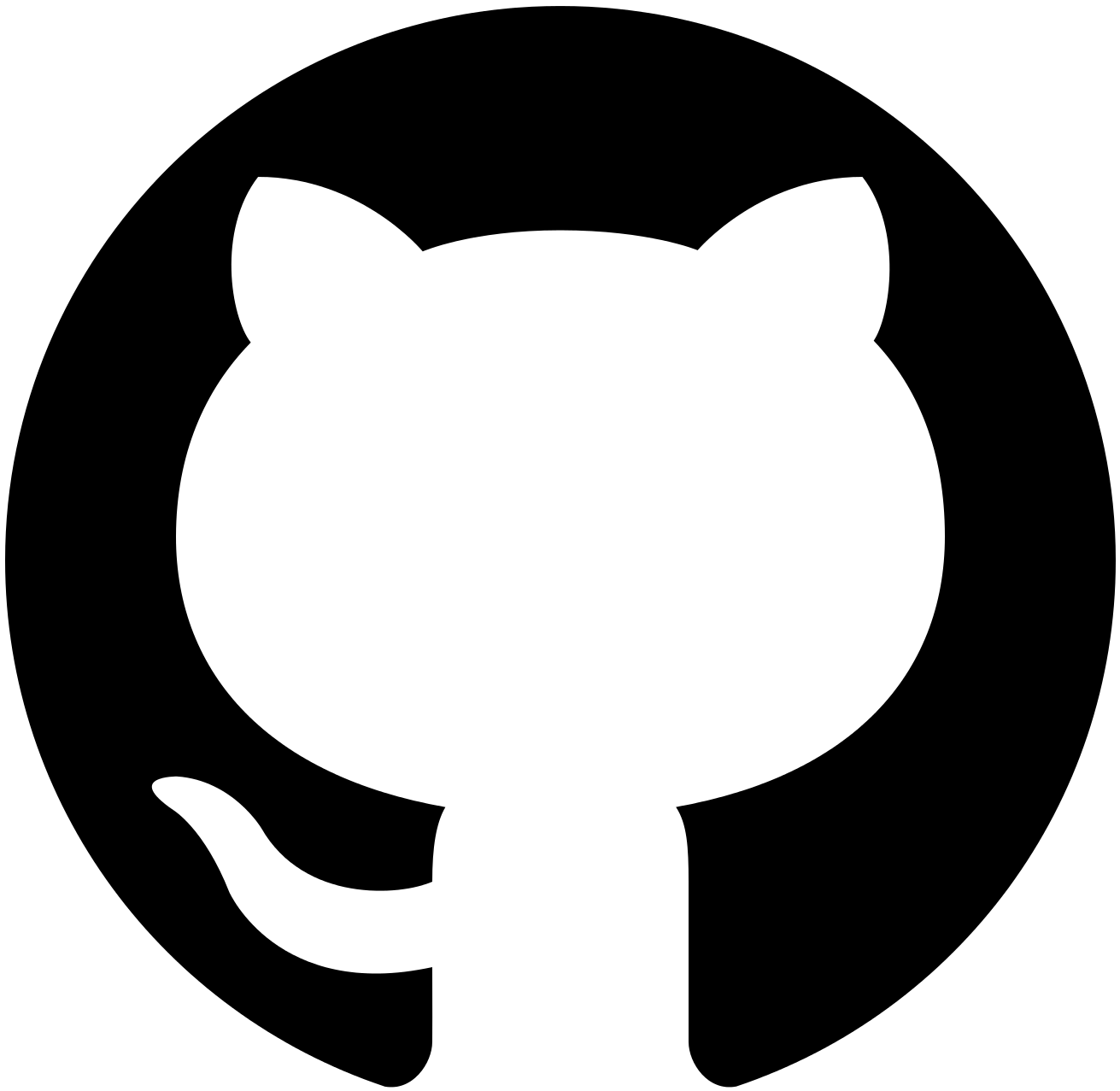}}~~ Project Website]{\href{https://mohamedgnabail.github.io/ubp2/}{\textbf{https://mohamedgnabail.github.io/ubp2/}}
}
\begin{document}

\maketitle
\blfootnote{$^*$ These authors contributed equally. Author order determined by coin flip.}

\begin{figure}[H]
    \centering
    \includegraphics[width=0.7\linewidth]{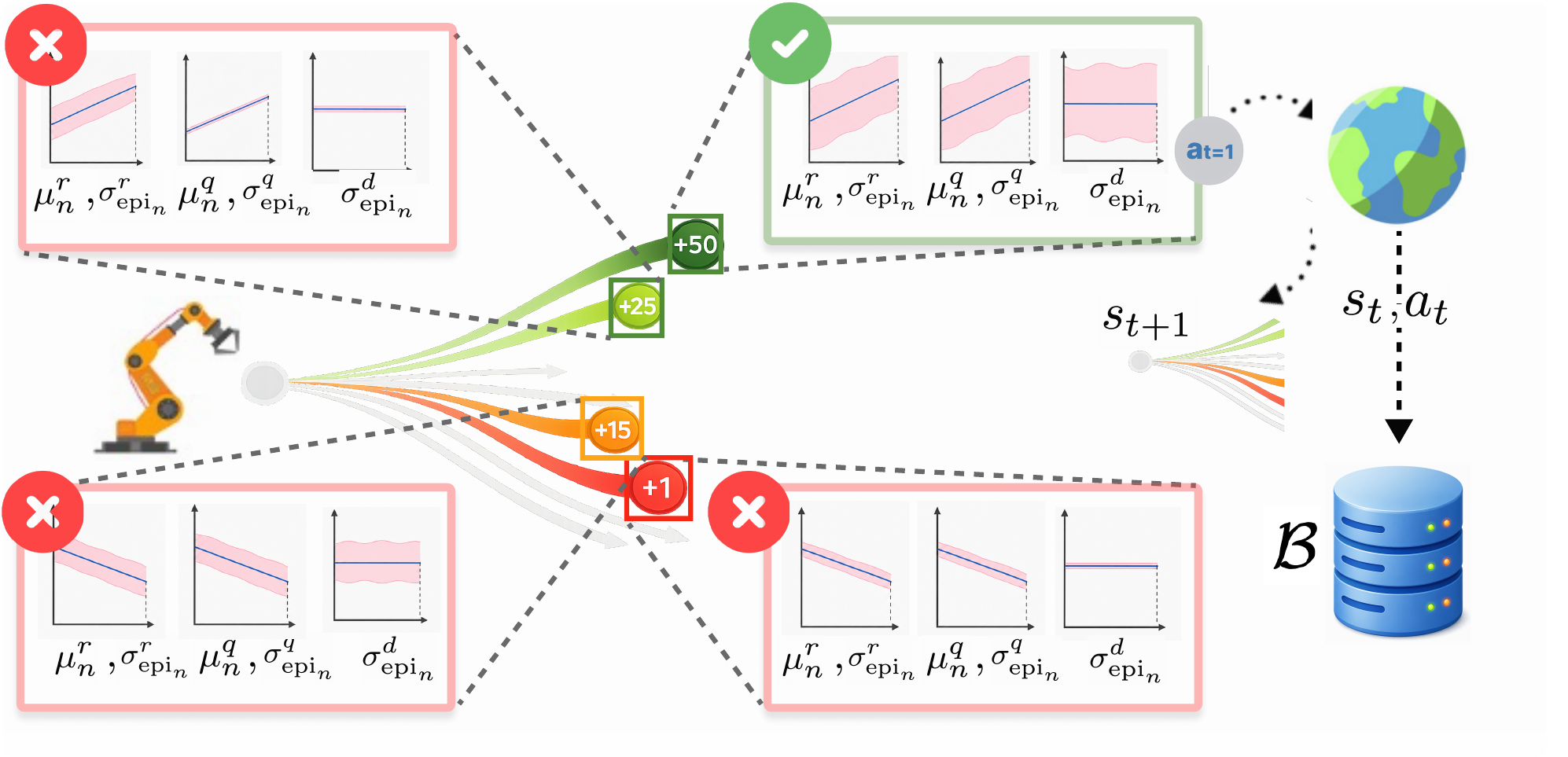}
    \caption{\small Illustration of planning in \algshort. The agent rolls out the dynamics model to generate imagined trajectories. The planner selects the trajectory maximizing expected return and epistemic uncertainty by solving \Cref{eq:plan_objective} and then executes $a_{t=1}$.}
    \label{fig:teaser}
  
\end{figure}

\section{Introduction}
In reinforcement learning, agents must learn effective behavior from interaction, often without prior knowledge of the dynamics of the environment or access to a densely specified reliable reward signal. This challenge is particularly acute in unsafe environments, where interactions are costly or risky, forcing agents to identify favorable actions and infer dynamics using as few samples as possible. Meanwhile, designing accurate reward functions for complex tasks is difficult and often leads to misspecified objectives that fail to capture the true intention \cite{abel2022expressivitymarkovreward,Singh2009Where}. Preference-based reward learning addresses this limitation by replacing manual reward design with pairwise trajectory comparisons provided by an oracle, such as a human \cite{christiano2017pbrl,wirth2017survey}. Instead of relying on explicit numerical rewards, the agent learns a reward model from preferences, which can then guide decision-making \cite{lee2021bpref}. However, effectively leveraging such learned rewards requires reasoning over hypothetical trajectories rather than relying solely on real environment interactions. Model-based reinforcement learning offers a natural solution by learning a predictive dynamics model and allowing planning over imagined trajectories, thereby improving sample efficiency, particularly in the early stages of learning \cite{moerland2022modelbasedreinforcementlearningsurvey,nagabandi2017neuralnetworkdynamicsmodelbased}.

In this paper, we propose an approach that learns environment dynamics and a reward model from preference feedback, and uses this learned reward as the only supervision to train a value function and policy. A model-based planner is used for efficient exploration and action selection under a limited preference feedback budget. After the budget is exhausted, control transitions to the learned policy.

In the context of preference-based reward learning, purely model-free schemes passively collect trajectories (often via random or heuristic exploration) and only afterward fit the reward model \cite{hejna2022fewshotpreferencelearninghumanintheloop, liang2022rewarduncertaintyexplorationpreferencebased, lee2021bpref}. In contrast, our model-based approach plans under the learned reward, dynamics, and value ensembles to actively select trajectories by maximizing a trajectory-level acquisition function that combines predicted cumulative return with epistemic uncertainty in the reward, dynamics, and value estimates. This objective jointly handles exploration and exploitation during learning. This tight coupling of planning and reward learning has been found to empirically accelerate convergence by focusing data collection on areas of interest.

We propose \alglong~(\algshort), which uses ensembles for reward, dynamics, and value models and plans trajectories using a unified score that combines expected return and uncertainty. Executing informative trajectories allows \algshort~to efficiently learn the dynamics model and to learn the reward from preferences. Uncertainty-guided planning is used only when preference feedback is available; once the feedback budget is exhausted, the agent switches to fast policy execution under the learned reward. Our experiments show that uncertainty-guided planning in the preference-based setting achieves earlier success than both non-optimistic model-based and model-free baselines in most tasks. We establish a sublinear regret bound for \algshort~in the number of environment interactions, under standard regularity assumptions. Our main contributions are:
\begin{itemize}
    \item We introduce an uncertainty-guided planning objective that combines predicted cumulative return with epistemic uncertainty over reward, dynamics, and value estimates, enabling optimistic exploration that outperforms planning based on return alone or uncertainty in a single model component.
     \item We introduce an optimistic preference label query strategy that prioritizes trajectory segment pairs with high predicted reward and high reward-model epistemic uncertainty.
    \item Under standard smoothness/RKHS assumptions and well-calibrated GP models, we establish finite- and infinite-horizon regret bounds for \algshort\ that are sublinear in the number of N episodes of environment interactions, with explicit dependence on the maximum information gain of the learned dynamics and reward models.
    \item We empirically demonstrate on Meta-World manipulation \cite{yu2021metaworldbenchmarkevaluationmultitask} tasks that \algshort\, operating on proprioceptive observations, achieves \emph{on average} improved sample-efficiency compared to state of the art model-free \cite{liang2022rewarduncertaintyexplorationpreferencebased} \cite{liu2022metarewardnet} and non-optimistic model-based preference baselines.
\end{itemize}
\section{Related Work}
We now compare our method to areas of closely-related prior work. UBP2 is a preference-based reinforcement learning approach that relies on uncertainty-driven exploration. 

\subsection{Preference-based Reinforcement Learning}

Although comparisons provide an intuitive form of feedback \citep{christiano2017pbrl}, achieving high sample-efficiency in human feedback is a challenge \cite{hejna2022fewshotpreferencelearninghumanintheloop}. Furthermore, collecting environment interactions for online learning can be impractical, particularly in high-risk environments. To address feedback sample-efficiency, prior work explored smart querying strategies, such as selecting the most informative queries through disagreement sampling after collecting an initial batch of seed interactions \cite{lee2021bpref}. Few-Shot PbRL \cite{hejna2022fewshotpreferencelearninghumanintheloop} uses a pretrained reward model on similar tasks for meta-learning the reward of the target task.
\cite{metcalf2022rewards} uses a pre-collected dataset of transitions to encode environment dynamics in reward learning. MoP-RL \cite{liu2023efficient} handles learning efficiency by relying on pre-collected human demonstrations and a pretrained dynamics model to improve sample-efficiency before reward learning begins. In contrast, our method, to the best of our knowledge, is the first model-based PbRL approach to learn the reward function entirely from scratch with no precollected agent interactions, no pretrained reward models, and no offline pretrained dynamics model. This setup forces the agent to learn purely through online environment interactions and preference feedback, providing a more realistic test of feedback sample efficiency.

\subsection{Epistemic uncertainty in model-based RL}
Prior work on Model-based RL estimate multiple sources of uncertainty for different intentions. MOPO estimates epistemic uncertainty in learned dynamics and penalizes the reward proportionally to this uncertainty, to discourage rollouts that leave the dataset support \cite{yu2020mopomodelbasedofflinepolicy}. RUNE, however, uses epistemic reward uncertainty optimistically as an intrinsic exploration bonus to improve feedback efficiency in PbRL \cite{liang2022rewarduncertaintyexplorationpreferencebased}. UNISafe treats epistemic uncertainty of the dynamics model as a proxy for Out Of Distribution (OOD) risk, so uncertainty is used to synthesize a safety filter \cite{seo2025uncertaintyawarelatentsafetyfilters}. \cite{feng2023finetuningofflineworldmodels} uses epistemic uncertainty in the value function as a regularizer of test-time planning, so candidate action sequences that appear high return but uncertain are de-prioritized, reducing extrapolation error during online finetuning. In contrast to using uncertainty in a single component of the world model, \algshort~plans with a unified trajectory score that combines predicted reward, terminal value, and epistemic uncertainty of the reward, dynamics, and value models, explicitly balancing exploration and exploitation in PbRL in a single planning objective.
\section{Preliminaries}
We consider a discrete-time Markov decision process (MDP) defined by the tuple $(\mathcal{S}, \mathcal{A}, p, r, \gamma, \rho_0)$, where $s\in\mathcal{S}$ with dimension $d_s$ and $a\in\mathcal{A}$ with dimension $d_a$ denote the continuous state and action spaces, $p:\mathcal{S}\times\mathcal{A}\mapsto \mathcal{S}$ is the unknown  transition distribution that captures process noise, $r:\mathcal{S}\times\mathcal{A}\mapsto \mathbb{R}$ is a generic reward function, $\rho_0$ is the initial state distribution and $\gamma \in (0,1)$ is the discount factor. The transition dynamics $p^*(s' \mid s, a)$ is assumed to be unknown. A parametric dynamics model $p_\phi(s' \mid s, a)$ is learned from data collected through online interaction with the environment and used during planning.

During model-based planning, the Model Predictive Control (MPC) planner generates the control actions using the models learned after the $n^{th}$ episode of interaction with the real environment, $n\in\{1,...,N\}$. This implicitly defines a policy through the solution of the planning objective, of which its first action is executed in the real environment; we denote this induced policy by $\pi^\mathrm{ind}_n$. In addition, we learn an explicit parametric policy $\pi^{\mathrm{learn}}_{\varphi}$, trained using \Cref{eq:policy_objective}.

The objective is to learn a control policy $\pi:\mathcal{S}\mapsto \mathcal{A}$ that maximizes the expected discounted return over $H$ timesteps under real dynamics $p^*$ and the true reward $r^*$ denoted by $\eta$, despite the absence of direct true reward observations. This desired policy, denoted $\pi^*$, is:
{ \[
\pi^\star
    =
    \argmax_{\pi} 
    \; \eta(\pi,p^*)
 =
    \argmax_{\pi}
    \; \mathbb{E}_{\pi,p^*}
    \left[
        \sum_{t=0}^{H-1} \gamma^t r^*(s_t, a_t)
    \right].\]
    }

\subsection{Preference-Based Reward Learning} \label{sec:pbrl}
In the PbRL setting, the agent does not observe samples from the ground-truth reward function $r^*(s,a)$.Instead, learning is guided by preference-based feedback in the form of pairwise comparisons between trajectory segments, where segments are uniformly sampled over the entire trajectory, rather than restricted to any particular sub-trajectory. Because preference queries are available only on a limited budget, their acquisition must be carefully optimized through informative exploration.

Given trajectory segments $\tau_i$ and $\tau_j$ each of length $L$,
a preference oracle indicates which segment is preferred. We model preferences using a learned reward function $r_\theta(s,a)$. Reward model parameters, $\theta$, are optimized by maximizing the log-likelihood of observed preference labels using the Bradley-Terry model, in which the probability that $\tau_i$ is preferred over $\tau_j$ is $P(\tau_i \succ \tau_j) =
\frac{\exp\big(\sum_{t=0}^{L-1} r_\theta(s_t^i, a_t^i)\big)}
{\exp\big(\sum_{t=0}^{L-1} r_\theta(s_t^i, a_t^i)\big) +
 \exp\big(\sum_{t=0}^{L-1} r_\theta(s_t^j, a_t^j)\big)}$.
\subsection{Modeling True Reward in PbRL}
In PbRL, preference data cannot uniquely identify $r^*$, as the oracle only reveals comparisons, i.e. information about which segment has larger relative utility, not the actual reward scale \cite{wirth2017survey}. This creates an \emph{equivalence class} of rewards - a set of functions that are \emph{indistinguishable from the perspective of the preference oracle} because they induce the exact same preference outcomes. Rather than claiming that PbRL recovers the literal, unobserved $r^*$, we interpret $r^*$ as a latent utility function that generates preferences and contend that the learning procedure produces some \emph{proxy reward}, denoted $r(s,a)$, within the same equivalence class as $r^*$, and hence is the subject of our study instead. We list further conditions on this equivalence class in Section \ref{sec:theory-results}, Assumption \ref{ass:RKHS-regularity}, which makes our theoretical analysis feasible. For more discussion of modeling the true reward in PbRL, please refer to Appendix \ref{sec:definitions}. Unless otherwise noted, henceforth the ``reward'' we refer to in all subsequent methodology and theory denotes the proxy reward $r$. The notation $r^*$ and the phrase ``true reward'' are reserved specifically when referring to the unobserved true reward.
\section{\algshort: \alglong~}
We now present \algshort, (Algorithm \ref{alg:main_algo}), our approach for efficient optimistic exploration in model-based reinforcement learning (MBRL). \algshort~uses ensembles of reward, dynamics, and value function models to guide exploration using epistemic uncertainty, enabling preference-based learning that is both sample-efficient and informative.

\textbf{Uncertainty Representation}: The uncertainty in the dynamics and value function models is estimated through ensemble disagreement, following previous work \cite{kidambi2021morelmodelbasedoffline,yu2020mopomodelbasedofflinepolicy}, which reflects total uncertainty that encompasses both epistemic and aleatoric effects. For the reward model, we adopt the Jensen–Rényi divergence (JRD) \cite{Renyi1961measures} to estimate epistemic uncertainty, following \cite{seo2025uncertaintyawarelatentsafetyfilters}. JRD explicitly separates epistemic and aleatoric uncertainty and has been shown to outperform total-uncertainty and density-based estimates in downstream tasks \cite{seo2025uncertaintyawarelatentsafetyfilters}. In our setting, we apply JRD for measuring epistemic uncertainty only to the reward model, and use disagreement based total uncertainty  for the dynamics and value function ensembles. We found doing so to be empirically sufficient for our results.
Our approach is based on TD-MPC2 \cite{hansen2024tdmpc}, using MPC as the primary control mechanism during the preference-based reward learning phase to allow efficient exploration under uncertainty. Unlike TD-MPC2  our approach does not assume access to scalar rewards throughout training, rather learns the reward from preference feedback and transition to a learned policy initialized from the value function once the preference budget is exhausted and model uncertainty decreases.
\vspace{-5pt}
\subsection{Uncertainty-Guided Optimistic Planning}\label{sec:uncertainty-guided-planning}
At each decision step $t$ in the real environment during the $n^{th}$ episode, \algshort~optimizes the parameters $(\mu_{\mathrm{mpc}}, \sigma_{\mathrm{mpc}})$ of a time-dependent Gaussian distribution over action sequences by approximately solving \Cref{eq:plan_objective}.

{\setlength{\abovedisplayskip}{1pt}
\begin{equation}
\label{eq:plan_objective}
\resizebox{0.96\linewidth}{!}{$
\displaystyle
\begin{aligned}
\argmax_{\mu_{\mathrm{mpc}},\,\sigma_{\mathrm{mpc}}}\;
\mathbb{E}_{\mathbf{a}_{t}} \Bigg[
\sum_{t=0}^{H-1} \gamma^{t}
\Big(
\mu^{r}_{n}(\hat{s}_t, a_t) 
+ \lambda_r \sigma^{r}_{\mathrm{epi}_{n}}(\hat{s}_t, a_t)
+ \lambda_d \sigma^{d}_{\mathrm{epi}_{n}}(\hat{s}_t, a_t)
\Big) + \gamma^{H}
\Big(
\mu^{q}_{n}(\hat{s}_H, a_H)
+ \lambda_q \sigma^{q}_{\mathrm{epi}_{n}}(\hat{s}_H, a_H)
\Big)
\Bigg],
\end{aligned}
$}
\end{equation}
}

where {\small\mbox{$\{a_t\}_{t=0}^{H-1} \sim
\mathcal{N}(\mu_{\mathrm{mpc}}, \sigma_{\mathrm{mpc}}^{2})$}} and {\small\mbox{$a_H \sim \pi^{\mathrm{learn}}(\cdot \mid \hat{s}_H)$}}.

Here, $(\mu_{\mathrm{mpc}}, \sigma_{\mathrm{mpc}})$ parameterize the planner’s time-dependent action distribution, while $\mu_{n}^{(\cdot)}$ and $\sigma^{(\cdot)}_{\mathrm{epi}_{n}}$ denote ensemble means and epistemic uncertainties of the learned reward, dynamics, and value models in the $n^{th}$ episodes of actual environment interaction (hence, these means are fixed for each planning procedure).
During planning, the trajectory $(\hat{s}_1,\dots,\hat{s}_H)$ is generated by rolling out the learned dynamics model conditioned on sampled action sequences $p_\phi(\hat{s}' \mid \hat{s}, a)$.
Only the first action $a_0$ of the optimized sequence is executed in the real environment, after which the planning process is repeated at the next timestep. For  long horizon reasoning beyond the planning horizon, we learn a value function $q_\psi(s,a)$ using temporal-difference learning under the learned reward  using $r_\theta(s,a)$. 

The coefficients $\lambda_r$, $\lambda_d$, and $\lambda_q$ control the relative contribution of the reward, the dynamics, and the value uncertainty, respectively, and are automatically tuned online. We use the ``autotuning'' approach from \cite{sukhija2025maxinforlboostingexplorationreinforcement} to optimize each coefficient using the current policy $\pi$ and a Polyak-averaged target policy $\bar\pi$ to define the loss: 
{
\begin{equation}
\label{eq:autotune}
\mathcal{L}(\lambda)
\!=\!
\mathbb{E}_{s \sim \mathcal{B}}
\Big[
\mathrm{log}(\lambda) \big(
\sigma_{epi}^{(.)}(s,\pi^{\mathrm{}}(s))\!-\!\sigma_{epi}^{(.)}(s,\bar\pi^{\mathrm{}}(s))
\big)
\Big].
\end{equation}}

\noindent
\begin{figure*}[t]
\centering
\nointerlineskip\vspace{-4pt}
\makebox[\textwidth]{%
% ================= RIGHT: ALGORITHM =================
\begin{minipage}[t]{0.49\textwidth}
\vspace{0pt}
\captionsetup{
  type=algorithm,
  justification=raggedright,
  singlelinecheck=false,
  labelfont=bf,
  skip=0pt,
  belowskip=0pt,
  aboveskip=0pt
}
\noindent\rule{\linewidth}{0.8pt}\par
\vspace{0.15\baselineskip}
\captionof{algorithm}{\algshort} \label{alg:main_algo}
\vspace{-0.55\baselineskip}
\noindent\rule{\linewidth}{0.4pt}\par
\vspace{-0.15\baselineskip}

\begin{algorithmic}[1]
\STATE {\bfseries Input:} Replay buffer $\mathcal{B}$, preference buffer $\mathcal{P}$
\STATE {\bfseries Params:} model ensembles $(p_\phi, r_\theta, q_\psi)$, policy $\pi_\varphi^{\mathrm{learn}}$
\STATE {\bfseries Hyperparameters:} horizon $H$, total steps after $N$ episodes $NT$, pref period $F_{\mathrm{pref}}$, pref budget $N_{\mathrm{pref}}$
\STATE Initialize $t\leftarrow 0$, $n_{\mathrm{pref}}\leftarrow 0$,$s\leftarrow \mathrm{env.reset}$
\WHILE{$t \le NT$}
   \IF{ ($n_{\mathrm{pref}} < N_{\mathrm{pref}}$)}
    \STATE  $a \leftarrow \mathrm{MPC}(s;\lambda_r,\lambda_d,\lambda_q)$
  \ELSE
    \STATE $a \leftarrow \pi_\varphi^{\mathrm{learn}}(s)$
  \ENDIF
  \STATE Step env: $(s', r) \leftarrow \mathrm{env.step}(a)$; 
  \STATE $\mathcal{B} \leftarrow \mathcal{B}\cup {(s,a)}$ ; $s\leftarrow s'$, $t\leftarrow t+1$

  \STATE Update $ p_\phi , r_\theta , q_\psi $ using \Cref{eq:cons_loss,eq:pref_loss,eq:q_loss}
  \STATE Update $\pi_\varphi^{\mathrm{learn}}$ using \Cref{eq:policy_objective}
  \STATE Update $\{\lambda_r,\lambda_d,\lambda_q\}$ using \Cref{eq:autotune}

 \IF{episode done (t==T)}
    \IF{$F_{\mathrm{pref}}$ and ($n_{\mathrm{pref}} < N_{\mathrm{pref}})$}
      \STATE $\mathcal{P} \leftarrow \mathcal{P} \cup {(\tau_1,\tau_2, y)}$ using  \Cref{alg:optimistic_pref}; 
      \STATE update $n_{\mathrm{pref}}$
    \ENDIF
    \STATE $s \leftarrow \mathrm{env.reset()}$
  \ENDIF
\ENDWHILE
\end{algorithmic}

\vspace{-0.65\baselineskip}
\noindent\rule{\linewidth}{0.4pt}
\end{minipage}
\hspace{0.5\fill}
% ================= LEFT: FIGURE =================
\begin{minipage}[t]{0.49\textwidth}
\vspace{0pt}
\centering
\includegraphics[width=\textwidth]{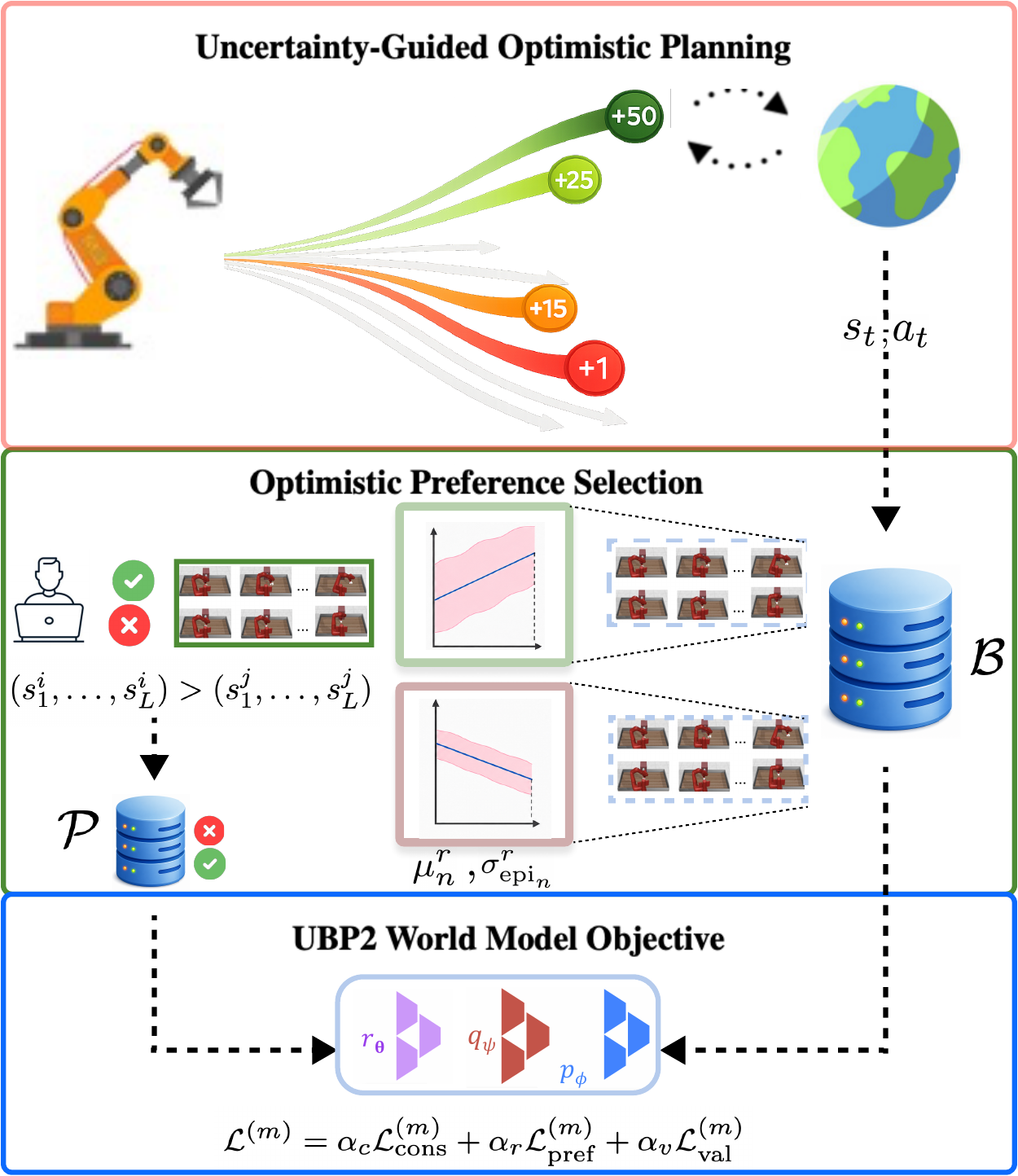}
\vspace{-10pt} 
\captionof{figure}{ \small \algshort~ uses ensembles of reward, dynamics, and value models to perform optimistic planning over imagined trajectories by combining predicted returns with epistemic uncertainty. Preferences are then selected optimistically from uniformly sampled segments of executed trajectories. Both trajectories and preferences are used to jointly train the dynamics, reward, and value models.}
\label{fig:method}
\end{minipage}
}
\vspace{-10pt}
\end{figure*}
\subsection{Optimistic Preference Selection}
Prior preference-based RL methods use query strategies such as uniform sampling, disagreement or entropy-based uncertainty, and coverage-based diversity \cite{lee2021bpref}. While these approaches encourage exploration, they are largely reward-agnostic and select queries from locally sampled candidate batches.
In contrast, our labeling approach (\Cref{alg:optimistic_pref}) differs in two ways: \textbf{1) Optimistic scoring:} we combine predicted preference likelihood with reward uncertainty to prioritize informative, high-value queries that improve downstream planning performance. \textbf{2) Global selection:} instead of sampling from local batches, we rank candidate trajectory pairs globally across the replay buffer, reducing the chance of missing highly informative comparisons.
\subsection{UBP2 World Model Objective}
\label{sec:ubp2_model_objective}
\paragraph{Model objective.}
UBP2 is based on the TD-MPC2 world-model objective, with the following key differences: (i) rewards are learned from preference feedback, (ii) TD targets are computed using predicted rewards, and (iii) reward and dynamics models are represented as ensembles.

Given a horizon segment $H$ and a preference pair segment $L$, we optimize a loss per-member that combines consistency, reward, and value losses,
{
\begin{equation*}
\mathcal{L}^{(m)} =  \alpha_c \mathcal{L}^{(m)}_{\mathrm{cons}} + \alpha_r \mathcal{L}^{(m)}_{\mathrm{pref}} + \alpha_v \mathcal{L}^{(m)}_{\mathrm{val}}.
\end{equation*}
}
The ensemble members $m$ of the dynamics and rewards models are independently trained as in \cite{seo2025uncertaintyawarelatentsafetyfilters}. To optimize the dynamics model, we compute a rollout $\{\hat s^{(m)}_t\}$ and minimize the consistency loss,
\vspace{-3pt}
{
\begin{equation}
\label{eq:cons_loss}
\mathcal{L}^{(m)}_{\mathrm{cons}} = \tfrac{1}{H}\sum_{t=0}^{H-1} \|\hat s^{(m)}_{t+1}-s_{t+1}\|_2^2 ,
\end{equation}
}
For the reward model, we minimize the preference loss, 
{
\begin{equation}
\begin{aligned}
\label{eq:pref_loss}
&\mathcal{L}_{\mathrm{pref}}^{(m)}=-\mathbb{E}_{(\tau_i,\tau_j,y)\sim\mathcal{P}}  \Big[y \log P_{\theta^{(m)}}(\tau_i\!\succ\!\tau_j)
\!+\!
(1\!-\!y)\log\!\big(1 \!-\! P_{\theta^{(m)}}(\tau_i\!\succ\!\tau_j)\big)
\Big],
\end{aligned}
\end{equation}
}
where $y$ is the true preference label provided by a preference oracle, which in our experiments is simulated using the true reward from the environment.

The value function ensemble is trained using temporal-difference regression. For each member $m$, the TD targets $ td_{t}$ are computed using the predicted reward of the same member,
{
  \begin{equation*}
  td_t^{(m)} = r^{(m)}_{\theta}(s_t,a_t) + \gamma  \bar q_{\min}(s_{t+1},\pi^{\mathrm{learn}}(s_{t+1})),
  \end{equation*}
  }
where $\bar q_{\min}$ denotes the target value function.
The value loss is:
\vspace{-3pt}
{
\begin{equation}
\label{eq:q_loss} 
\mathcal{L}_{\mathrm{val}}^{(m)} =
\frac{1}{H E_Q}
\sum_{t=0}^{H-1}\gamma^{t}
\sum_{j=1}^{E_Q}
\Big[
\big(q^{(j)}_{\psi}(\hat s_t,a_t) - td_t^{(m)}\big)^2
\Big].
\end{equation}
}
We adopt the stochastic maximum-entropy policy objective from TD-MPC2.
The policy $\pi_\varphi(a\mid s)$ is optimized to maximize value while encouraging
exploration through entropy regularization over imagined rollouts. The policy objective is given in \Cref{eq:policy_objective}.
\vspace{-3pt}
{
\setlength{\abovedisplayskip}{1pt}
\setlength{\belowdisplayskip}{1pt}
\setlength{\abovedisplayshortskip}{1pt}
\setlength{\belowdisplayshortskip}{1pt}
\begin{align}
\label{eq:policy_objective}
\mathcal{L}(\varphi)
\!=\!-\frac{1}{H}
\sum_{t=0}^{H-1}
\gamma^{t}
\big(
\mu^{q}(\hat s_t, \pi_\varphi^{\mathrm{}}( \hat s_t))
\!+\!
\alpha_{\mathcal{H}}\,\mathcal{H}(\pi_\varphi^{\mathrm{}})
\big),
\end{align}} where $\mathcal{H}(\pi_\varphi^{\mathrm{}})$ denotes the policy entropy. Gradients of $\mathcal{L}_{\pi}$ are taken with respect to the policy parameters $\varphi$ only, with value estimates fixed during policy optimization.
\section{Theoretical Results}  \label{sec:theory-results}
 First, we make standard continuity common in literature \cite{curi2020efficientmodelbasedreinforcementlearning, sussex2023modelbasedcausalbayesianoptimization,sukhija2023optimisticactiveexplorationdynamical}.

\begin{assumption}[Continuous dynamics, continuous and bounded reward]\label{ass:cont-dyn-and-r} The true dynamics $p^*$, all policies $\pi\in \Pi$, and true reward $r^*$ are continuous. Furthermore, $r^*$ is bounded, i.e. $r^*:\mathcal{S}\times \mathcal{A}\mapsto[0,R_{\text{max}}]$, and the process noise, captured by $p^*(s'|s,a)$, is i.i.d. Gaussian with variance $\sigma^2$.
\end{assumption}

After the $n^{th}$ episode of environment interaction, we learn a  model of the true dynamics, with a mean $\mu^d_n(s,a)$ and uncertainty estimate $\sigma^{d}_{\mathrm{epi}_{n}}(s,a)$, and also a model of the proxy reward with a mean $\mu^r_n(s,a)$ and uncertainty $\sigma^{r}_{\mathrm{epi}_{n}}(s,a)$. Like previous work, we model each dynamics coordinate $p^*_j,\:j\in \{1,...,d_s\}$ using a Gaussian process (GP) \cite{sukhija2025sombrlscalableoptimisticmodelbased}. The scalar reward $r$ is treated analogously. For clarity, we limit our discussion on the setting with Gaussian noise, though this is not a requirement for the final regret bound. Please refer to Appendix \ref{sec:sub-gaussian-noise} for settings with subgaussian process noise.
\begin{assumption}[RKHS regularity]\label{ass:RKHS-regularity} The ground truth dynamics $p^*$ and the proxy reward $r$ reside in their own respective Reproducing Kernel Hilbert Spaces (RKHS). Each dynamics coordinate $p_j^*$ lies in an RKHS $\mathcal{H}_{k_d}$ induced by kernel $k_d$ with bounded norm $\|p_j^*\|_{\mathcal{H}_{k_d}}\le B_d$, and $k_d$ is uniformly bounded along the diagonal, i.e.\ $\sup_{(s,a)\in\mathcal{S}\times\mathcal{A}} k_d\big((s,a),(s,a)\big)\le \sigma_{\max}^d$. Furthermore, we assume $r\in \mathcal{H}_{k_r}$ with bounded norm $\|r\|_{\mathcal{H}_{k_r}}\le B_r$ and with $k_r$ uniformly bounded along the diagonal, i.e.\ $\sup_{(s,a)\in\mathcal{S}\times\mathcal{A}} k_r\big((s,a),(s,a)\big)\le \sigma_{\max}^r$.
\end{assumption}

Assumption \ref{ass:RKHS-regularity}, common in the Bayesian optimization and RL literature \cite{srinvas2012, curi2020efficientmodelbasedreinforcementlearning, sukhija2025sombrlscalableoptimisticmodelbased}, allows the modeling and learning of the true dynamics and reward using GPs. This admits a closed form for the posterior mean and variance of the learned models, presented in Appendix \ref{sec:definitions}.

Given dynamics $p$ we define the mean action-value function $\bar{Q}$ after the $n^{th}$ episode of environment interaction to be
{
$$
\bar{Q}^\pi_{p,n}(s,a) := \mathbb{E}_{\pi,p}\!\left[\sum_{t=0}^{\infty} \gamma^t\, \mu_n^r(s_t,a_t)\middle| s_0=s,a_0=a\right],
\label{action_value_definition_mean}
$$
}
and denote our estimated action-value function as $\hat{Q}^\pi_n$. Following a similar approach as Theorem 1 in \cite{sikchi2021learningoffpolicyonlineplanning}, we have:
\begin{assumption} [$Q$-function sub-optimality] 
     At the $n^{th}$ episode of environment interaction, the estimation error of $Q$ is bounded by a constant $\epsilon_{q,n}$, i.e.,
    \begin{equation}
        \lVert \hat{Q}_n^{\pi}-\bar{Q}^{\pi}_{p,n}\rVert_\infty\leq\epsilon_{q,n}.
    \end{equation}
    \label{ass:q-suboptimality}
\end{assumption}
To quantify the epistemic uncertainty of $Q$, we have the following assumption:

\begin{assumption}[Epistemic Uncertainty of $Q$]\label{ass:q-uncertainty}
For every episode $n$, the epistemic uncertainty $\sigma_n^q$ of the action--value function $\bar Q_n^{\pi}$ is proportional to the discounted cumulative epistemic uncertainty arising from the dynamics and reward models, i.e.,
{\[
\sigma_n^q \;\propto\;
\mathbb{E}_{\pi,p}\!\left[\sum_{t=0}^{\infty}\gamma^{t}\,\lVert\sigma_n^d(s_t,a_t)\rVert\right]
\;+\;
\mathbb{E}_{\pi,p}\!\left[\sum_{t=0}^{\infty}\gamma^{t}\,\sigma_n^r(s_t,a_t)\right].
\]}
\end{assumption}

This assumption reflects that uncertainty in the action--value function accumulates from epistemic uncertainty in both the dynamics and reward models along trajectories induced by policy $\pi$. A detailed theoretical justification and discussion are provided in Appendix~\ref{sec:inf-bound-proof}.

For the purposes of theoretical analysis, we provide the following definition of the UBP planner objective that is optimized by the planner at each planning procedure using fixed learned dynamics and reward models $\mu_n^{(\cdot)}$, $\sigma_{\text{epi}_n}^{(\cdot)}$:
\begin{definition}[Optimistic planning objective] Let $\pi$ denote a candidate policy (i.e. the policy at the current planning iteration) and consider rollouts using the mean dynamics $\mu_n^{d}$ with the \emph{same} process noise as the true environment. We define the UCB planning objective over planner horizon $H$:
{
\begin{equation}
\eta_{\gamma,n}^{\mathrm{UCB}}(\pi,\mu_n^d)
:= \mathbb{E}_{\pi,\mu^d_{n}}
\Bigg[
\sum_{t=0}^{H-1}\gamma^t
\Big(
\mu^r_{n}(\hat s_t,a_t)
+ a_n \sigma^r_{n}(\hat s_t,a_t)
+ b_n \sigma^d_{n}(\hat s_t,a_t)
\Big)+\gamma^H\Big( \mu_n^q(\hat{s}_t,a_t)+c_n\sigma_n^q(\hat{s}_t,a_t)\Big)
\Bigg],\label{eq:theory-plan-objective} 
\end{equation}
}
where $a_n,b_n,c_n$ are parameters that control the strength of reward, dynamics exploration, with their theoretical bounds provided in Lemma \ref{lem:inf-ucb-optimism} and justified in Appendix \ref{sec:inf-bound-proof}. They correspond to hyperparameters $\lambda_r,\lambda_d, \lambda_q$ in \Cref{eq:plan_objective}.
\end{definition}

To derive a bound, we define the \emph{cumulative infinite-horizon regret} for a particular policy $\pi_n$ as

\noindent
\begin{equation}
\label{eq:inf-regret}
R_{\gamma,N}
=\sum_{n=1}^{N}r_{\gamma,n}
=\sum_{n=1}^{N}\eta_\gamma(\pi^*,p^*)-\eta_\gamma(\pi_n,p^*),
\end{equation}
where

\begin{equation}
\label{eq:inf-return}
\eta_\gamma(\pi,p)
=\mathbb{E}_{\pi,p}\!\left[\sum_{t=0}^\infty \gamma^t r(s_t,a_t)\right].
\end{equation}

We make the following assumption on the planner:
\begin{assumption}[No planner sub-optimality] \label{ass:no-plan-subopt}The policy $\pi^{\text{ind}}_n$ \emph{induced} by the planning procedure through optimizing \eqref{eq:theory-plan-objective} is equal to the policy $\pi^*_n$ that maximizes \eqref{eq:theory-plan-objective}. That is, $\pi^{\text{ind}}_n=\pi^*_n$.

We refer to the policy as being \emph{induced} by the planner because the planner does not train a parameterized policy. Rather, it repeatedly solves \eqref{eq:theory-plan-objective} and executes only the first action. Note that this is not the learned parametric policy $\pi^{\text{learn}}_\varphi$. For theoretical results that consider planner suboptimality, please refer to Appendix \ref{sec:lemma-proofs-plan-sub}. We now present the following lemma, which states: 
\end{assumption}

\begin{lemma}[Optimism of the UCB planning objective]\label{lem:inf-ucb-optimism} Let Assumptions \ref{ass:cont-dyn-and-r}, \ref{ass:RKHS-regularity}, \ref{ass:q-suboptimality}, and \ref{ass:q-uncertainty} hold, and let the learned GP dynamics and reward models be well-calibrated. Then there exist coefficients $a_n\in\mathcal{O}\!\big(\sqrt{\Gamma_{r,N}}\big)$, 
$b_n\in\mathcal{O}\!\big(\sqrt{\Gamma_{d,N}}\big)$ and $c_n\in\mathcal{O}\!\big(\max\{\sqrt{\Gamma_{r,N}},\sqrt{\Gamma_{d,N}}\}\big)$ such that
for all episodes $n$ and all policies $\pi$, with probability at least $1-\delta$:
\begin{equation}
\eta_\gamma(\pi,p^*) \;\le\; \eta_{\gamma,n}^{\mathrm{UCB}}(\pi,\mu_n^d)+\gamma^H\epsilon_{q,n}.
\end{equation}
\end{lemma}

$\Gamma_{d,N}$ is the state and action dimension-dependent maximum information gain \cite{srinvas2012, sukhija2025sombrlscalableoptimisticmodelbased}, defined in Appendix \ref{sec:definitions}, for a function in the RKHS induced by $k_d$ over $N$ episodes of environment interaction and similarly for $\Gamma_{r,N}$. Lemma \ref{lem:inf-ucb-optimism} states that the planning objective upper-bounds the true return uniformly over policies up to some $Q$ estimation error, i.e., the agent plans using an optimistic estimate of the environment. We first establish an analogous result for a reduced objective in Appendix~\ref{sec:lemma-proofs}. The proof of Lemma~\ref{lem:inf-ucb-optimism} builds on this reduced case and is provided in Appendix~\ref{sec:inf-bound-proof}.

We now present the following theorem, which bounds the regret of our method:

\begin{theorem} [Infinite-horizon regret bound with dynamics, reward and value uncertainty bonuses] \label{theorem:inf-regret-bound}
    Let Assumptions \ref{ass:cont-dyn-and-r}, \ref{ass:RKHS-regularity}, \ref{ass:q-suboptimality}, \ref{ass:q-uncertainty} and \ref{ass:no-plan-subopt} hold, and let the learned GP dynamics and reward models be well-calibrated. The cumulative infinite-horizon regret after $N$ episodes of real environment interaction satisfies, with probability $1-\delta$:

{
\centering
\noindent
\begin{minipage}{0.65\linewidth}
\begin{equation}
R_{\gamma,N}
\!\le\!
\mathcal{O}\;\Big(
H^3 \sqrt{N}\,
\big(\Gamma_{d,N\log N}^{3/2}\!+\!\Gamma_{r,N\log N}^{3/2}\big)\!+\!e_{q,N}
\Big),
\label{eq:main-regret-inf}
\end{equation}
\end{minipage}
\hfill
\begin{minipage}{0.25\linewidth}
where $e_{q,N}:=\sum_{n=1}^N\epsilon_{q,n}$.
\end{minipage}
}
\end{theorem}

 If we additionally assume that the growth rate of $e_{q,N}$ is sublinear, then Theorem \ref{theorem:inf-regret-bound} yields a sublinear regret rate governed by the maximum information gain of the dynamics and reward kernels, and parallels the infinite-horizon GP regret guarantees established for past work \cite{sukhija2025sombrlscalableoptimisticmodelbased}. Our additional dependence on the decomposition into $(\Gamma_{d},\Gamma_{r})$ arise from separately controlling
(1) the simulation error due to transition uncertainty and
(2) reward estimation uncertainty across the planning horizon. As in previous work, we first establish an analogous sublinear result for a reduced finite-horizon objective in Appendix \ref{sec:lemma-proofs} and then prove Theorem \ref{theorem:inf-regret-bound} in Appendix \ref{sec:inf-bound-proof}. Note that our regret bound as written is implicit in the state and action dimension $d_s$ and $d_a$; it depends on the information gain in order to present it in a kernel-agnostic manner. Please refer to Appendix \ref{sec:inf-bound-proof} for the explict regret bound for different kernels.

\section{Experiments}
\label{sec:experiments}
\begin{figure*}[t]
\footnotesize
\centering
% ---------------- Row 1 ----------------
\makebox[\textwidth]{%
\subfigure[Door Close (500)]{%
    \begin{minipage}{0.198\textwidth}
    \centering
        \includegraphics[width=\textwidth]{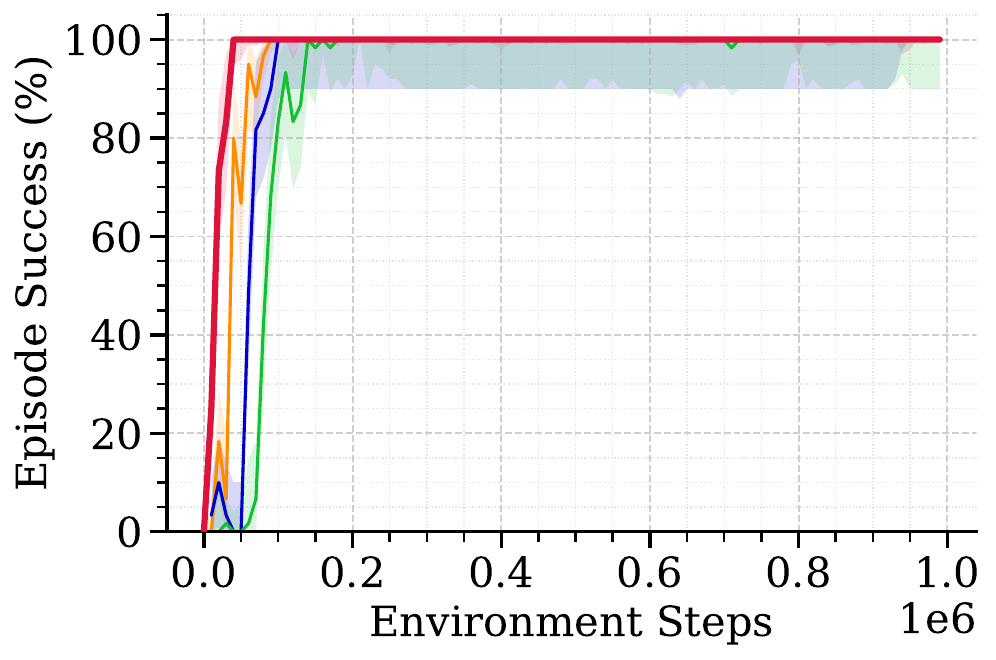}
    \end{minipage}
}\hfill
\subfigure[Window Close (500)]{%
    \begin{minipage}{0.198\textwidth}
    \centering
        \includegraphics[width=\textwidth]{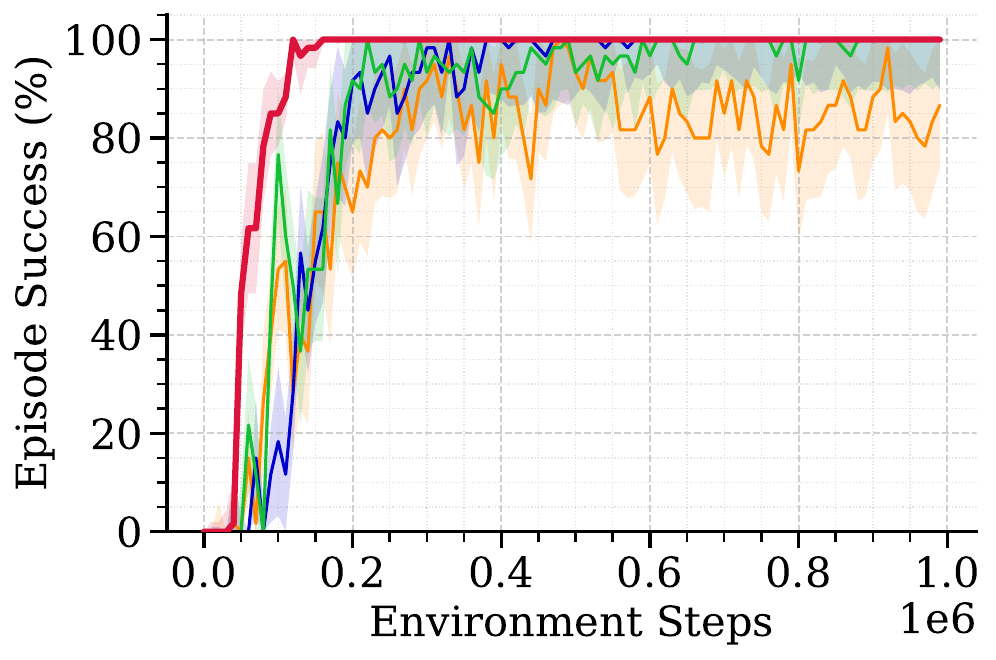}
    \end{minipage}
}\hfill
\subfigure[Handle Press (1000)]{%
    \begin{minipage}{0.198\textwidth}
    \centering
        \includegraphics[width=\textwidth]{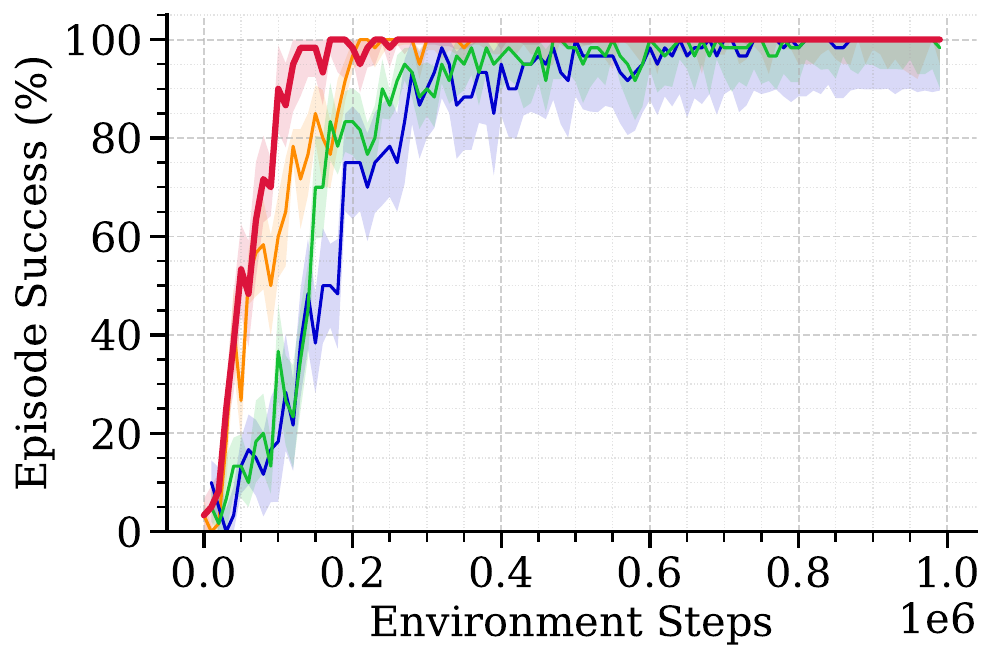}
    \end{minipage}
}\hfill
\subfigure[Coffee Button (1000)] {%
    \begin{minipage}{0.198\textwidth}
    \centering
        \includegraphics[width=\textwidth]{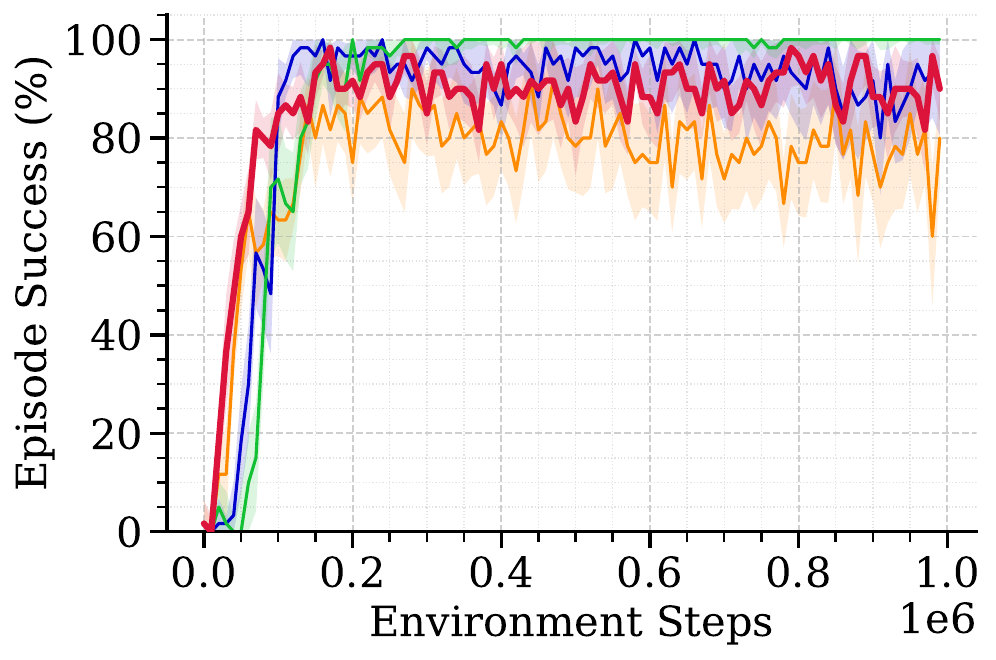}
    \end{minipage}
}\hfill

\subfigure[Faucet Open (2000)]{%
    \begin{minipage}{0.198\textwidth}
    \centering
        \includegraphics[width=\textwidth]{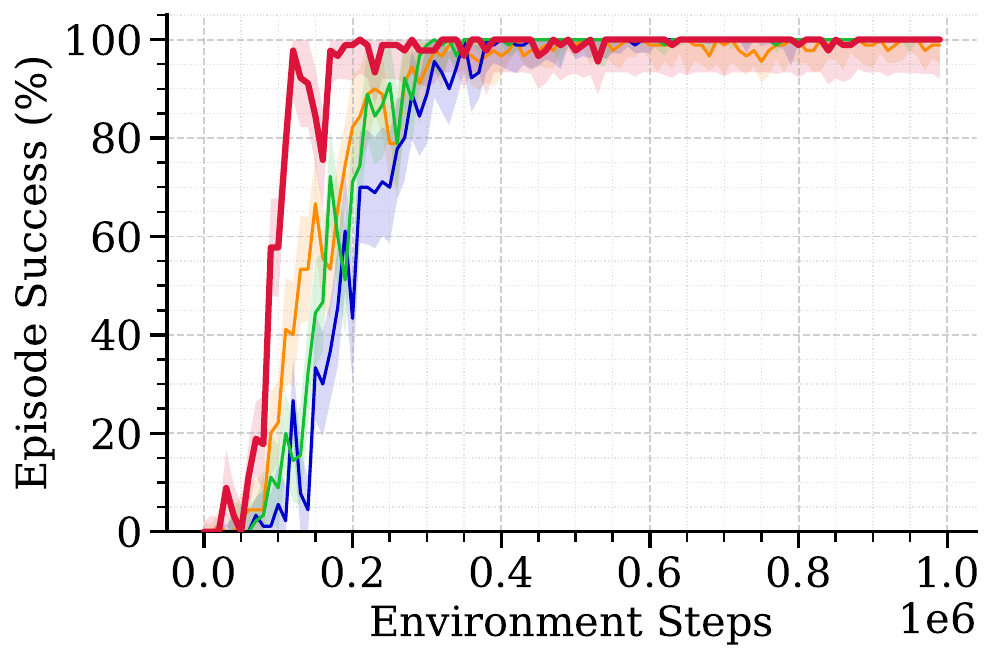}
    \end{minipage}
}\hfill
}
% ---------------- Row 2 ----------------
\makebox[\textwidth]{%
\subfigure[Door Open (2000)]{%
    \begin{minipage}{0.198\textwidth}
    \centering
        \includegraphics[width=\textwidth]{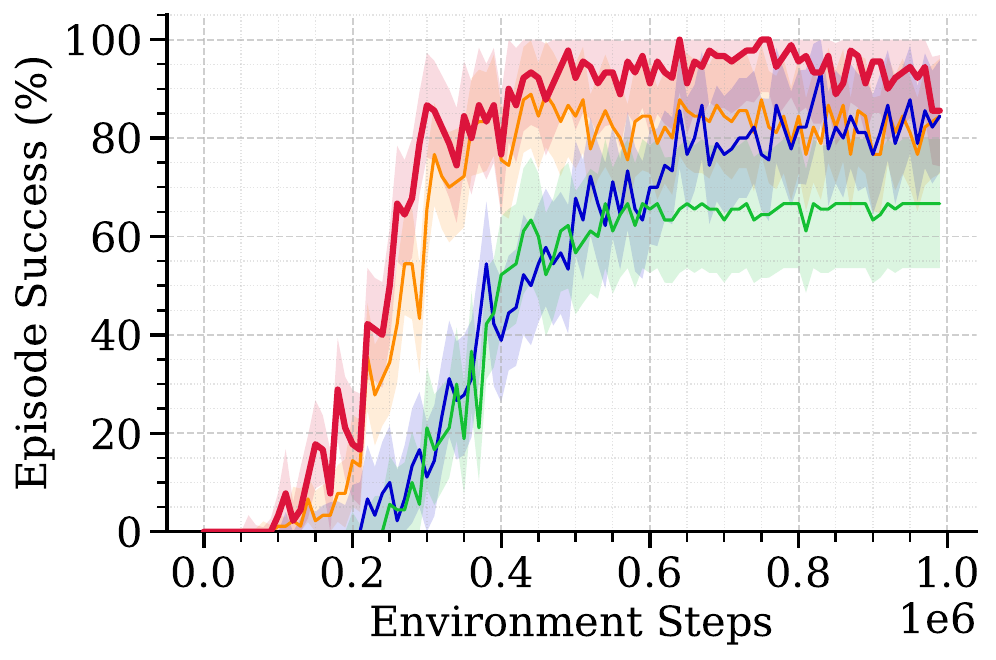}
    \end{minipage}
}\hfill
\subfigure[Door Unlock (2500)]{%
    \begin{minipage}{0.198\textwidth}
    \centering
        \includegraphics[width=\textwidth]{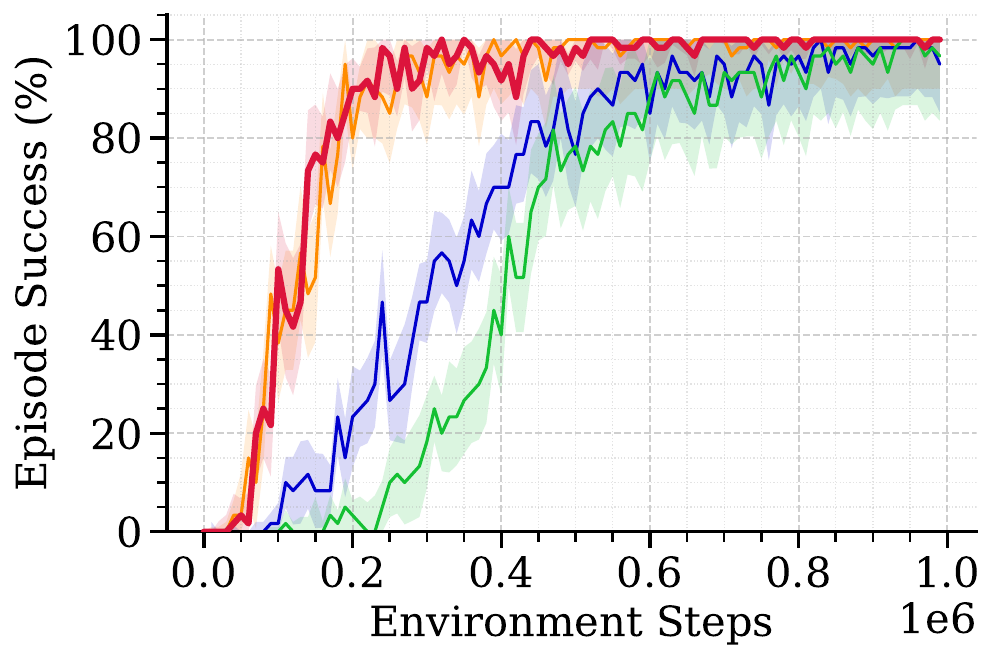}
    \end{minipage}
}\hfill
\subfigure[Sweep Into (5000)] {%
    \begin{minipage}{0.198\textwidth}
    \centering
        \includegraphics[width=\textwidth]{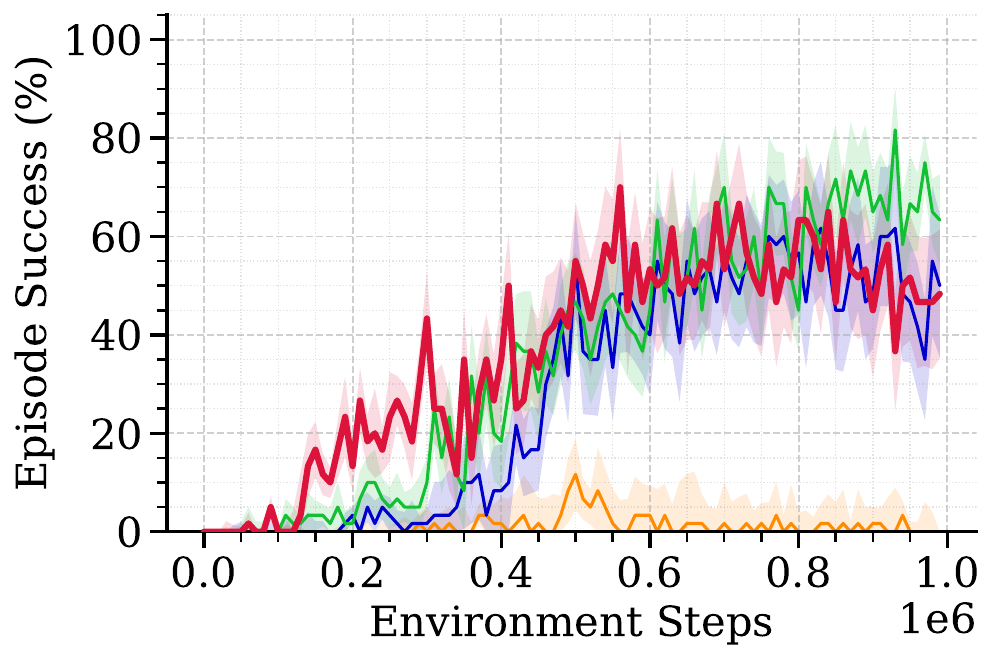}
    \end{minipage}
    }
\hfill
\subfigure[ Drawer Open (5000)]{%
    \begin{minipage}{0.198\textwidth}
    \centering
        \includegraphics[width=\textwidth]{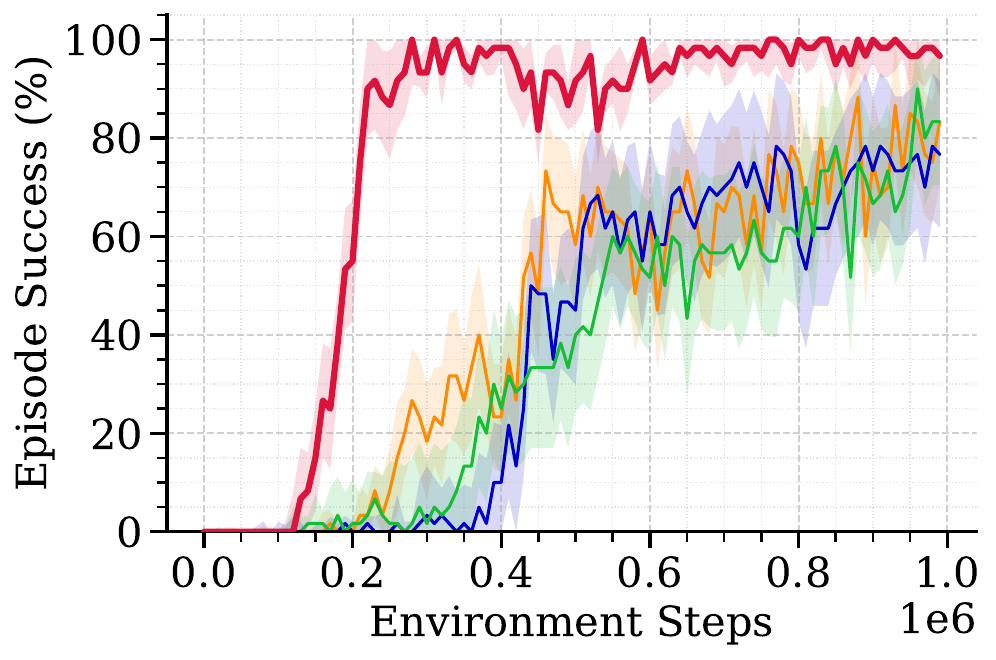}
    \end{minipage}
    }
\hfill
\subfigure[{\footnotesize Hammer (10000)}]{%
    \begin{minipage}{0.198\textwidth}
    \centering
        \includegraphics[width=\textwidth]{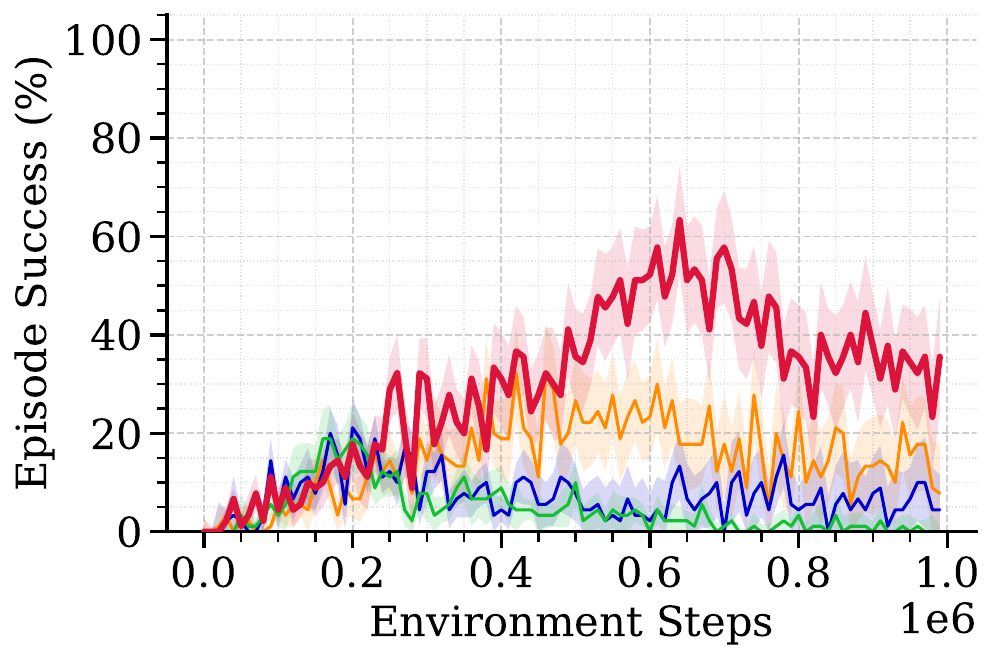}
    \end{minipage}
}\hfill

}
% ---------------- Legend ----------------
\vspace{-10pt}
\includegraphics[width=0.4\textwidth]{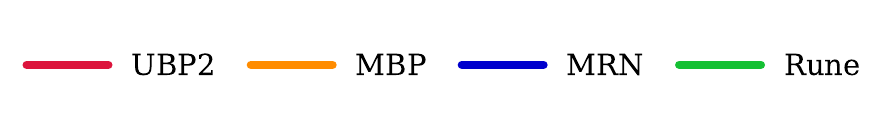}
\vspace{-5pt}
\caption{
Performance comparison across MetaWorld tasks.
Curves show IQM performance with shaded standard error.
Numbers between brackets indicate the maximum preference feedback budget.}
\label{fig:main_results}
\end{figure*}
We design our experiments to answer the following:
\textbf{Q1:} Does uncertainty-guided optimistic planning improve sample efficiency in preference-based reinforcement learning compared to model-free and non-optimistic model-based baselines? \textbf{Q2:} How do uncertainty-aware planning components (dynamics, reward, and value), together with optimistic preference selection, individually affect success-rate performance? \textbf{Q3:} Can ~\algshort~ benefit from varying the planning horizon? \textbf{Q4:} How does \algshort~ perform under constrained preference feedback budgets?  
Overall, we find that our approach is substantially more sample-efficient than strong model free approaches, and that all components of uncertainty are generally helpful to boosting sample efficiency and robustness.

\begin{table*}[h]
\centering
\caption{\small Performance comparison on MetaWorld tasks over 1M environment steps. We report IQM success averaged over seeds and time, with standard error at the final evaluation step, along with average performance and average rank across tasks. Best results are shown in bold. Asterisks denote statistically significant improvements of \algshort~ over the corresponding baseline under a Wilcoxon signed-rank test (p < 0.05), reported only when sufficient seeds are available.}
\label{tab:main_result}
\resizebox{\textwidth}{!}{
\begin{tabular}{lllllllllll|>{\columncolor{blue!15}}l|>{\columncolor{blue!15}}l}
\toprule
\textbf{Method} & \textbf{Door Close} & \textbf{Window Close} & \textbf{Handle Press} & \textbf{Coffee Button} & \textbf{Faucet Open} & \textbf{Door Open} & \textbf{Door Unlock} & \textbf{Sweep Into} & \textbf{Drawer Open} & \textbf{Hammer} & \textbf{Avg $\uparrow$} & \textbf{Rank $\downarrow$} \\
\midrule
UBP2 & \textbf{97.8\,({\footnotesize 1.0})} & \textbf{93.1\,({\footnotesize 1.0})} & \textbf{93.4\,({\footnotesize 0.0})} & 86.2\,({\footnotesize 3.7}) & \textbf{89.5\,({\footnotesize 7.9})} & \textbf{70.7\,({\footnotesize 10.4})} & \textbf{85.8\,({\footnotesize 2.1})} & \textbf{36.4\,({\footnotesize 12.3})} & \textbf{77.7\,({\footnotesize 4.9})} & \textbf{29.9\,({\footnotesize 11.5})} & \textbf{76.1} & \textbf{1.20} \\
\midrule
MBP & 95.6\,({\footnotesize 0.0})\text{*} & 74.7\,({\footnotesize 13.5})\text{*} & 90.8\,({\footnotesize 4.0})\text{*} & 75.3\,({\footnotesize 10.7}) & 83.6\,({\footnotesize 5.0})\text{*} & 61.5\,({\footnotesize 11.7}) & 85.2\,({\footnotesize 10.0}) & 1.1\,({\footnotesize 2.0})\text{*} & 43.7\,({\footnotesize 12.8})\text{*} & 14.8\,({\footnotesize 9.1}) & 62.6 & 2.60 \\
Rune \cite{liang2022rewarduncertaintyexplorationpreferencebased} & 91.5\,({\footnotesize 10.0})\text{*} & 85.2\,({\footnotesize 9.2})\text{*} & 84.5\,({\footnotesize 8.1})\text{*} & \textbf{90.4\,({\footnotesize 1.0})} & 83.1\,({\footnotesize 1.4})\text{*} & 41.7\,({\footnotesize 13.1})\text{*} & 56.1\,({\footnotesize 13.1})\text{*} & 35.5\,({\footnotesize 9.3}) & 35.9\,({\footnotesize 12.8})\text{*} & 4.6\,({\footnotesize 0.9})\text{*} & 60.8 & 3.10 \\
MRN \cite{liu2022metarewardnet}  & 94.2\,({\footnotesize 0.0})\text{*} & 84.9\,({\footnotesize 10.2})\text{*} & 81.5\,({\footnotesize 10.3})\text{*} & 87.7\,({\footnotesize 11.0}) & 80.1\,({\footnotesize 0.0})\text{*} & 48.2\,({\footnotesize 11.4}) & 65.6\,({\footnotesize 9.8})\text{*} & 28.3\,({\footnotesize 14.5}) & 38.1\,({\footnotesize 14.9})\text{*} & 7.5\,({\footnotesize 7.2})\text{*} & 61.6 & 3.10 \\
\bottomrule
\end{tabular}
}
\end{table*}

\begin{table*}[h]
\centering
\caption{\small Ablation study on MetaWorld tasks over 1M environment steps. Config indicates whether Reward, Dynamics, Value Uncertainty, and Optimistic Preference Selection are enabled (1) or disabled (0). We report IQM success averaged over seeds and time, with standard error at the final evaluation step. We also report overall mean performance and worst-case regret (maximum per-task drop relative to the best variant). Best results are in bold.}
\label{tab:ablation_tasks}
\resizebox{\textwidth}{!}{
\begin{tabular}{llllllllllll|>{\columncolor{blue!15}}l|>{\columncolor{blue!15}}l}
\toprule
\textbf{Approach} & \textbf{Config} & \textbf{Door Close} & \textbf{Window Close} & \textbf{Handle Press} & \textbf{Coffee Button} & \textbf{Faucet Open} & \textbf{Door Open} & \textbf{Door Unlock} & \textbf{Sweep Into} & \textbf{Drawer Open} & \textbf{Hammer} & \textbf{Avg $\uparrow$} & \textbf{Regret $\downarrow$} \\
\midrule
UBP2 & 1111 & \textbf{97.8\,({\footnotesize 1.0})} & 93.1\,({\footnotesize 1.0}) & 93.4\,({\footnotesize 0.0}) & 86.2\,({\footnotesize 3.7}) & 89.5\,({\footnotesize 7.9}) & 70.7\,({\footnotesize 10.4}) & 85.8\,({\footnotesize 2.1}) & 36.4\,({\footnotesize 12.3}) & \textbf{77.7\,({\footnotesize 4.9})} & 29.9\,({\footnotesize 11.5}) & 76.1 & 13.8 \\
\midrule
 & 0111 & 96.7\,({\footnotesize 2.0}) & \textbf{94.4\,({\footnotesize 0.0})} & 91.6\,({\footnotesize 0.0}) & \textbf{90.4\,({\footnotesize 9.7})} & 82.8\,({\footnotesize 4.0}) & \textbf{84.5\,({\footnotesize 2.4})} & 85.0\,({\footnotesize 0.0}) & \textbf{38.9\,({\footnotesize 16.4})} & 68.3\,({\footnotesize 4.0}) & 22.4\,({\footnotesize 19.7}) & 75.5 & 15.8 \\
 & 1001 & 94.8\,({\footnotesize 0.0}) & 93.1\,({\footnotesize 0.0}) & 91.2\,({\footnotesize 0.0}) & 82.8\,({\footnotesize 11.7}) & \textbf{90.3\,({\footnotesize 0.0})} & 67.7\,({\footnotesize 14.0}) & 78.4\,({\footnotesize 0.0}) & 20.3\,({\footnotesize 13.2}) & 44.0\,({\footnotesize 19.6}) & 14.1\,({\footnotesize 18.1}) & 67.7 & 33.7 \\
 & 1011 & 95.9\,({\footnotesize 0.0}) & 94.2\,({\footnotesize 0.0}) & 92.0\,({\footnotesize 0.0}) & 84.2\,({\footnotesize 8.7}) & 89.7\,({\footnotesize 2.4}) & 58.5\,({\footnotesize 24.5}) & 83.3\,({\footnotesize 0.0}) & 29.3\,({\footnotesize 13.2}) & 58.4\,({\footnotesize 19.6}) & 12.4\,({\footnotesize 10.7}) & 69.8 & 26.0 \\
 & 0101 & 97.5\,({\footnotesize 0.0}) & 93.6\,({\footnotesize 0.0}) & 91.0\,({\footnotesize 20.0}) & 76.8\,({\footnotesize 9.7}) & 85.6\,({\footnotesize 20.0}) & 82.4\,({\footnotesize 16.0}) & 81.1\,({\footnotesize 0.0}) & 31.4\,({\footnotesize 22.3}) & 53.9\,({\footnotesize 18.9}) & 15.1\,({\footnotesize 19.4}) & 70.8 & 23.8 \\
 & 1101 & 96.4\,({\footnotesize 20.0}) & 94.2\,({\footnotesize 0.0}) & 92.1\,({\footnotesize 20.0}) & 90.1\,({\footnotesize 3.2}) & 84.7\,({\footnotesize 6.0}) & 78.8\,({\footnotesize 0.0}) & \textbf{88.6\,({\footnotesize 0.0})} & 4.0\,({\footnotesize 15.6}) & 76.0\,({\footnotesize 4.0}) & \textbf{38.2\,({\footnotesize 22.6})} & 74.3 & 34.9 \\
 & 0011 & 97.3\,({\footnotesize 8.0}) & 94.0\,({\footnotesize 0.0}) & \textbf{93.7\,({\footnotesize 0.0})} & 89.8\,({\footnotesize 8.0}) & 80.0\,({\footnotesize 23.3}) & 66.9\,({\footnotesize 2.0}) & 83.6\,({\footnotesize 4.0}) & 15.0\,({\footnotesize 15.7}) & 71.8\,({\footnotesize 2.0}) & 33.3\,({\footnotesize 19.7}) & 72.5 & 23.9 \\
 & 1110 & 96.0\,({\footnotesize 0.0}) & 89.8\,({\footnotesize 0.0}) & 88.8\,({\footnotesize 0.0}) & 79.5\,({\footnotesize 12.4}) & 87.7\,({\footnotesize 0.0}) & 68.9\,({\footnotesize 8.0}) & 76.0\,({\footnotesize 6.3}) & 11.2\,({\footnotesize 9.7}) & 68.2\,({\footnotesize 16.0}) & 27.4\,({\footnotesize 19.6}) & 69.3 & 27.7 \\
\midrule
MBP & 0000 & 95.6\,({\footnotesize 0.0}) & 74.7\,({\footnotesize 13.5}) & 90.8\,({\footnotesize 4.0}) & 75.3\,({\footnotesize 10.7}) & 83.6\,({\footnotesize 5.0}) & 61.5\,({\footnotesize 11.7}) & 85.2\,({\footnotesize 10.0}) & 1.1\,({\footnotesize 2.0}) & 43.7\,({\footnotesize 12.8}) & 14.8\,({\footnotesize 9.1}) & 62.6 & 37.8 \\
\bottomrule
\end{tabular}
}
\end{table*}

\textbf{Experimental Setup.} We evaluate \algshort\ on 10 manipulation tasks of varying degrees of complexity from the MetaWorld benchmark \cite{yu2021metaworldbenchmarkevaluationmultitask} using proprioceptive observations and task-specific success metrics \cite{oquab2024dinov2learningrobustvisual}. For our main results, we use feedback budgets commonly reported in the literature \cite{liu2022metarewardnet, liang2022rewarduncertaintyexplorationpreferencebased, lee2021bpref}. We compare \textsc{UBP2} with the following baselines: \textbf{(1) RUNE} \cite{liang2022rewarduncertaintyexplorationpreferencebased}: A model-free, preference-based reinforcement learning method that enhances exploration by augmenting the extrinsic reward with an intrinsic bonus based on disagreement among an ensemble of learned reward models , \textbf{(2) MRN (Meta-Reward-Net)} \cite{liu2022metarewardnet}: A model-free, preference-based reinforcement learning approach that employs bi-level optimization to learn the reward function based on the performance of the Q-function. \textbf{(3) MBP} (Model-Based PbRL): A non-optimistic variant of UBP2 that removes uncertainty-aware planning and optimistic preference selection. This baseline isolates the effect of model-based learning without optimism or uncertainty-driven exploration. This variant is conceptually similar to MoP-RL \cite{liu2023efficient}. \textbf{Evaluation Metric}: We report Interquartile Mean (IQM) success rate over environment steps  where faster gains indicate higher sample efficiency in preference-based learning. We report results using the IQM in our plots and tables, as a robust aggregate metric for reinforcement learning evaluation \cite{agarwal2022deepreinforcementlearningedge}. IQM reduces the influence of outlier runs and provides a more stable estimate compared to the mean, especially in the low-seed regime typical of computationally intensive RL experiments. In our experiments we run 5-15 seeds per task based on task complexity. 

\textbf{Sample Efficiency in PbRL}: \Cref{fig:main_results} shows that UBP2 reaches high success rates earlier in 9 out of the 10 tasks, demonstrating the benefit of optimistic planning in PbRL. \Cref{tab:main_result} further supports these findings by reporting IQM success averaged across evaluation steps over the course of training. UBP2 achieves both the highest overall average performance and the best average rank across tasks.

\textbf{Component Ablations}: We evaluate the contribution of each component in UBP2 through 5-seed ablation studies against the full model and MBP, where all optimistic components are disabled. The results in \Cref{tab:ablation_tasks} show that UBP2 provides the best overall balance between performance and robustness. Although some ablations slightly improve individual tasks, these gains are inconsistent and often incur much higher regret elsewhere. Since regret reflects worst-case degradation across tasks of varying difficulty, these results highlight the importance of each optimistic component to achieve strong and consistent performance. UBP2 also achieves the highest average overall performance, indicating improved sample efficiency and robustness across tasks

\textbf{Optimistic Preference Selection} We compare UBP2 against two commonly used preference-query heuristics discussed in prior work: disagreement-based and entropy-based selection.  \Cref{fig:ablation_pref_selection} shows that optimistic preference selection consistently outperforms both alternatives in challenging manipulation tasks

\textbf{Varying Horizon} We evaluate the effect of the planning horizon in \Cref{fig:ablation_horizon}, a key degree of freedom in model-based planning enabled by UBP2. Unlike model free methods, UBP2 allows adapting the planning horizon to better match task structure. Because the optimal horizon varies across tasks, we tune it separately for each environment. Across the three representative tasks shown, increasing the horizon from 7 to 11 consistently improves performance. Although even longer horizons may provide additional gains, they come at a significantly higher computational cost during planning. Refer to \Cref{subsec:vary_horizon} for more  experiments. 

\textbf{Reduced Feedback} We evaluate UBP2 under constrained preference feedback budget, 2 tasks receive a limited feedback budget. The results in   \Cref{fig:ablation_feedback} indicate that \algshort~ retains its effectiveness even when the available feedback is significantly reduced.

\begin{figure*}[t]
\footnotesize
\centering
% ---------------- Row 1 ----------------
\makebox[\textwidth]{%
\subfigure[Preference Selection Method]{%
    \label{fig:ablation_pref_selection}
    \begin{minipage}{0.31\textwidth}
    \centering
        \includegraphics[width=\textwidth]{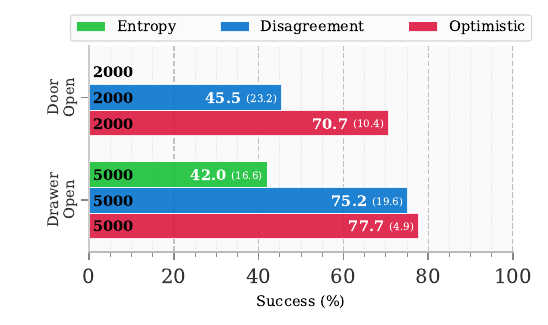}
    \end{minipage}
}\hfill
\subfigure[Planning Horizon]{%
    \label{fig:ablation_horizon}
    \begin{minipage}{0.31\textwidth}
    \centering
        \includegraphics[width=\textwidth]{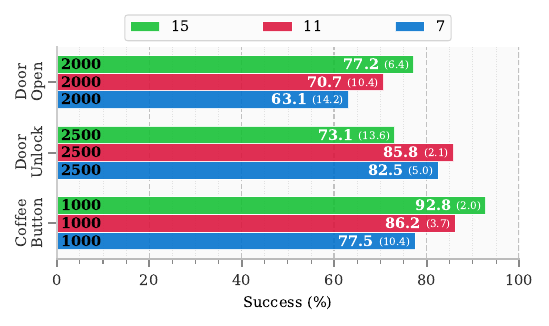}
    \end{minipage}
}\hfill
\subfigure[Reduced Feedback]{%
    \label{fig:ablation_feedback}
    \begin{minipage}{0.31\textwidth}
    \centering
        \includegraphics[width=\textwidth]{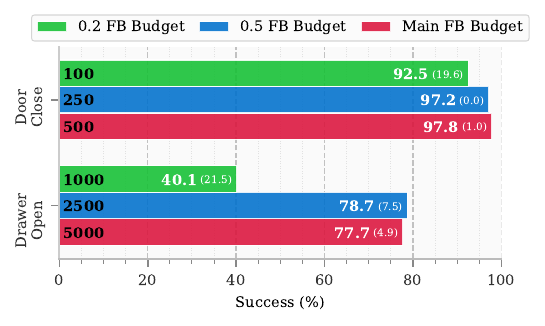}
    \end{minipage}
    
}\hfill
}
\caption{Sensitivity analysis of \algshort\ across three key hyperparameters on MetaWorld tasks.
\textbf{(a)} Preference selection strategy.
\textbf{(b)} Planning horizon.
\textbf{(c)} Feedback budget.
IQM success averaged over seeds and evaluation steps is reported in white, with standard error at the final evaluation step in brackets; feedback budget is reported in black on the left of each bar.
In~\textbf{(a)}, entropy-based preference selection achieves the lowest average IQM of $0.1\%$ {\footnotesize(20)} on Door Open.
}
\label{fig:ablation_pref_horizon_feedback}
\end{figure*}
\section{Conclusion}
This work presents \alglong~(\algshort), a model-based preference-based reinforcement learning approach that actively guides exploration by optimistically planning over uncertainty in the reward, dynamics, and value functions. By planning under a unified optimistic objective that balances expected return with epistemic uncertainty, \algshort~explicitly balances exploitation and information acquisition, avoiding uninformative data collection.
\vspace{-10pt}
\section{Limitations}
A limitation of our work is that the role of preference learning in the theoretical guarantees could be made more explicit. There exist past work that tightly characterizes preference-learning error via more stringent assumptions on the reward function \cite{zhan2024provablerewardagnosticpreferencebasedreinforcement}. Furthermore, the modeling of the dynamics and reward as well-calibrated GPs in the theory may not coincide with  practical implementation (i.e. deep ensembles). For more discussion regarding both these points, please refer to \Cref{sec:definitions}. 
\section{Acknowledgment}
This research was supported by an NSERC Discovery Grant, the NVIDIA Academic Grant Program, and computing resources from the Digital Research Alliance of Canada \href{https://alliancecan.ca}{(alliancecan.ca)}.

\bibliographystyle{plainnat}
\bibliography{main}

@misc{moerland2022modelbasedreinforcementlearningsurvey,
      title={Model-based Reinforcement Learning: A Survey}, 
      author={Thomas M. Moerland and Joost Broekens and Aske Plaat and Catholijn M. Jonker},
      year={2022},
      eprint={2006.16712},
      archivePrefix={arXiv},
      primaryClass={cs.LG},
      url={https://arxiv.org/abs/2006.16712}, 
}

@inproceedings{
lee2021bpref,
title={B-Pref: Benchmarking Preference-Based Reinforcement Learning},
author={Kimin Lee and Laura Smith and Anca Dragan and Pieter Abbeel},
booktitle={Thirty-fifth Conference on Neural Information Processing Systems Datasets and Benchmarks Track (Round 1)},
year={2021},
url={https://openreview.net/forum?id=ps95-mkHF_}
}

@misc{curi2020efficientmodelbasedreinforcementlearning,
      title={Efficient Model-Based Reinforcement Learning through Optimistic Policy Search and Planning}, 
      author={Sebastian Curi and Felix Berkenkamp and Andreas Krause},
      year={2020},
      eprint={2006.08684},
      archivePrefix={arXiv},
      primaryClass={cs.LG},
      url={https://arxiv.org/abs/2006.08684}, 
}

@misc{sukhija2023optimisticactiveexplorationdynamical,
      title={Optimistic Active Exploration of Dynamical Systems}, 
      author={Bhavya Sukhija and Lenart Treven and Cansu Sancaktar and Sebastian Blaes and Stelian Coros and Andreas Krause},
      year={2023},
      eprint={2306.12371},
      archivePrefix={arXiv},
      primaryClass={cs.LG},
      url={https://arxiv.org/abs/2306.12371}, 
}

@misc{chowdhury2017kernelizedmultiarmedbandits,
      title={On Kernelized Multi-armed Bandits}, 
      author={Sayak Ray Chowdhury and Aditya Gopalan},
      year={2017},
      eprint={1704.00445},
      archivePrefix={arXiv},
      primaryClass={cs.LG},
      url={https://arxiv.org/abs/1704.00445}, 
}

@article{deboer2005ce,
  title={A Tutorial on the Cross-Entropy Method},
  author={de Boer, Pieter-Tjerk and Kroese, Dirk P. and Mannor, Shie and Rubinstein, Reuven Y.},
  journal={Annals of Operations Research},
  year={2005},
  url={https://people.smp.uq.edu.au/DirkKroese/ps/aortut.pdf}
}

@misc{sambharya2022endtoendlearningwarmstartrealtime,
      title={End-to-End Learning to Warm-Start for Real-Time Quadratic Optimization}, 
      author={Rajiv Sambharya and Georgina Hall and Brandon Amos and Bartolomeo Stellato},
      year={2022},
      eprint={2212.08260},
      archivePrefix={arXiv},
      primaryClass={math.OC},
      url={https://arxiv.org/abs/2212.08260}, 
}

@misc{byravan2021evaluatingmodelbasedplanningplanner,
      title={Evaluating model-based planning and planner amortization for continuous control}, 
      author={Arunkumar Byravan and Leonard Hasenclever and Piotr Trochim and Mehdi Mirza and Alessandro Davide Ialongo and Yuval Tassa and Jost Tobias Springenberg and Abbas Abdolmaleki and Nicolas Heess and Josh Merel and Martin Riedmiller},
      year={2021},
      eprint={2110.03363},
      archivePrefix={arXiv},
      primaryClass={cs.RO},
      url={https://arxiv.org/abs/2110.03363}, 
}

@misc{ghadimi2013stochasticfirstzerothordermethods,
      title={Stochastic First- and Zeroth-order Methods for Nonconvex Stochastic Programming}, 
      author={Saeed Ghadimi and Guanghui Lan},
      year={2013},
      eprint={1309.5549},
      archivePrefix={arXiv},
      primaryClass={math.OC},
      url={https://arxiv.org/abs/1309.5549}, 
}

@misc{jeong2025reflectthenplanofflinemodelbasedplanning,
      title={Reflect-then-Plan: Offline Model-Based Planning through a Doubly Bayesian Lens}, 
      author={Jihwan Jeong and Xiaoyu Wang and Jingmin Wang and Scott Sanner and Pascal Poupart},
      year={2025},
      eprint={2506.06261},
      archivePrefix={arXiv},
      primaryClass={cs.AI},
      url={https://arxiv.org/abs/2506.06261}, 
}

@misc{argenson2021modelbasedofflineplanning,
      title={Model-Based Offline Planning}, 
      author={Arthur Argenson and Gabriel Dulac-Arnold},
      year={2021},
      eprint={2008.05556},
      archivePrefix={arXiv},
      primaryClass={cs.LG},
      url={https://arxiv.org/abs/2008.05556}, 
}

@misc{sikchi2021learningoffpolicyonlineplanning,
      title={Learning Off-Policy with Online Planning}, 
      author={Harshit Sikchi and Wenxuan Zhou and David Held},
      year={2021},
      eprint={2008.10066},
      archivePrefix={arXiv},
      primaryClass={cs.LG},
      url={https://arxiv.org/abs/2008.10066}, 
}

@misc{rothfuss2023hallucinatedadversarialcontrolconservative,
      title={Hallucinated Adversarial Control for Conservative Offline Policy Evaluation}, 
      author={Jonas Rothfuss and Bhavya Sukhija and Tobias Birchler and Parnian Kassraie and Andreas Krause},
      year={2023},
      eprint={2303.01076},
      archivePrefix={arXiv},
      primaryClass={cs.LG},
      url={https://arxiv.org/abs/2303.01076}, 
}

@ARTICLE{srinvas2012,
  author={Srinivas, Niranjan and Krause, Andreas and Kakade, Sham M. and Seeger, Matthias W.},
  journal={IEEE Transactions on Information Theory}, 
  title={Information-Theoretic Regret Bounds for Gaussian Process Optimization in the Bandit Setting}, 
  year={2012},
  volume={58},
  number={5},
  pages={3250-3265},
  doi={10.1109/TIT.2011.2182033}}

@misc{sukhija2025sombrlscalableoptimisticmodelbased,
      title={SOMBRL: Scalable and Optimistic Model-Based RL}, 
      author={Bhavya Sukhija and Lenart Treven and Carmelo Sferrazza and Florian Dörfler and Pieter Abbeel and Andreas Krause},
      year={2025},
      eprint={2511.20066},
      archivePrefix={arXiv},
      primaryClass={cs.LG},
      url={https://arxiv.org/abs/2511.20066}, 
}

@misc{sussex2023modelbasedcausalbayesianoptimization,
      title={Model-based Causal Bayesian Optimization}, 
      author={Scott Sussex and Anastasiia Makarova and Andreas Krause},
      year={2023},
      eprint={2211.10257},
      archivePrefix={arXiv},
      primaryClass={cs.LG},
      url={https://arxiv.org/abs/2211.10257}, 
}

@misc{abel2022expressivitymarkovreward,
      title={On the Expressivity of Markov Reward}, 
      author={David Abel and Will Dabney and Anna Harutyunyan and Mark K. Ho and Michael L. Littman and Doina Precup and Satinder Singh},
      year={2022},
      eprint={2111.00876},
      archivePrefix={arXiv},
      primaryClass={cs.LG},
      url={https://arxiv.org/abs/2111.00876}, 
}

@inproceedings{Singh2009Where,
  author    = {Satinder Singh and Richard L. Lewis and Andrew G. Barto},
  title     = {Where Do Rewards Come From?},
  booktitle = {Proceedings of the Annual Conference of the Cognitive Science Society},
  pages     = {2601--2606},
  year      = {2009},
  publisher = {Cognitive Science Society}
}

@misc{nagabandi2017neuralnetworkdynamicsmodelbased,
      title={Neural Network Dynamics for Model-Based Deep Reinforcement Learning with Model-Free Fine-Tuning}, 
      author={Anusha Nagabandi and Gregory Kahn and Ronald S. Fearing and Sergey Levine},
      year={2017},
      eprint={1708.02596},
      archivePrefix={arXiv},
      primaryClass={cs.LG},
      url={https://arxiv.org/abs/1708.02596}, 
}

@misc{zhan2024provablerewardagnosticpreferencebasedreinforcement,
      title={Provable Reward-Agnostic Preference-Based Reinforcement Learning}, 
      author={Wenhao Zhan and Masatoshi Uehara and Wen Sun and Jason D. Lee},
      year={2024},
      eprint={2305.18505},
      archivePrefix={arXiv},
      primaryClass={cs.LG},
      url={https://arxiv.org/abs/2305.18505}, 
}

@misc{janz2023banditoptimisationfunctionsmatern,
      title={Bandit optimisation of functions in the Mat\'ern kernel RKHS}, 
      author={David Janz and David R. Burt and Javier González},
      year={2023},
      eprint={2001.10396},
      archivePrefix={arXiv},
      primaryClass={cs.LG},
      url={https://arxiv.org/abs/2001.10396}, 
}

@misc{kakade2020informationtheoreticregretbounds,
      title={Information Theoretic Regret Bounds for Online Nonlinear Control}, 
      author={Sham Kakade and Akshay Krishnamurthy and Kendall Lowrey and Motoya Ohnishi and Wen Sun},
      year={2020},
      eprint={2006.12466},
      archivePrefix={arXiv},
      primaryClass={cs.LG},
      url={https://arxiv.org/abs/2006.12466}, 
}

@misc{liang2022rewarduncertaintyexplorationpreferencebased,
      title={Reward Uncertainty for Exploration in Preference-based Reinforcement Learning}, 
      author={Xinran Liang and Katherine Shu and Kimin Lee and Pieter Abbeel},
      year={2022},
      eprint={2205.12401},
      archivePrefix={arXiv},
      primaryClass={cs.LG},
      url={https://arxiv.org/abs/2205.12401}, 
}

@misc{hejna2022fewshotpreferencelearninghumanintheloop,
      title={Few-Shot Preference Learning for Human-in-the-Loop RL}, 
      author={Joey Hejna and Dorsa Sadigh},
      year={2022},
      eprint={2212.03363},
      archivePrefix={arXiv},
      primaryClass={cs.RO},
      url={https://arxiv.org/abs/2212.03363}, 
}

@article{christiano2017pbrl,
  title={Deep reinforcement learning from human preferences},
  author={Christiano, Paul F and Leike, Jan and Brown, Tom and Martic, Miljan and Legg, Shane and Amodei, Dario},
  journal={Advances in neural information processing systems},
  volume={30},
  year={2017}
}

@inproceedings{liu2023efficient,
  title={Efficient preference-based reinforcement learning using learned dynamics models},
  author={Liu, Yi and Datta, Gaurav and Novoseller, Ellen and Brown, Daniel S},
  booktitle={2023 IEEE International Conference on Robotics and Automation (ICRA)},
  pages={2921--2928},
  year={2023},
  organization={IEEE}
}

@inproceedings{
hansen2024tdmpc,
title={{TD}-{MPC}2: Scalable, Robust World Models for Continuous Control},
author={Nicklas Hansen and Hao Su and Xiaolong Wang},
booktitle={The Twelfth International Conference on Learning Representations},
year={2024},
url={https://openreview.net/forum?id=Oxh5CstDJU}
}

@article{metcalf2022rewards,
  title={Rewards encoding environment dynamics improves preference-based reinforcement learning},
  author={Metcalf, Katherine and Sarabia, Miguel and Theobald, Barry-John},
  journal={arXiv preprint arXiv:2211.06527},
  year={2022}
}

@article{wirth2017survey,
  title={A survey of preference-based reinforcement learning methods},
  author={Wirth, Christian and Akrour, Riad and Neumann, Gerhard and F{\"u}rnkranz, Johannes},
  journal={Journal of Machine Learning Research},
  volume={18},
  number={136},
  pages={1--46},
  year={2017}
}

@misc{seo2025uncertaintyawarelatentsafetyfilters,
      title={Uncertainty-aware Latent Safety Filters for Avoiding Out-of-Distribution Failures}, 
      author={Junwon Seo and Kensuke Nakamura and Andrea Bajcsy},
      year={2025},
      eprint={2505.00779},
      archivePrefix={arXiv},
      primaryClass={cs.RO},
      url={https://arxiv.org/abs/2505.00779}, 
}

@misc{yu2020mopomodelbasedofflinepolicy,
      title={MOPO: Model-based Offline Policy Optimization}, 
      author={Tianhe Yu and Garrett Thomas and Lantao Yu and Stefano Ermon and James Zou and Sergey Levine and Chelsea Finn and Tengyu Ma},
      year={2020},
      eprint={2005.13239},
      archivePrefix={arXiv},
      primaryClass={cs.LG},
      url={https://arxiv.org/abs/2005.13239}, 
}

@misc{kidambi2021morelmodelbasedoffline,
      title={MOReL : Model-Based Offline Reinforcement Learning}, 
      author={Rahul Kidambi and Aravind Rajeswaran and Praneeth Netrapalli and Thorsten Joachims},
      year={2021},
      eprint={2005.05951},
      archivePrefix={arXiv},
      primaryClass={cs.LG},
      url={https://arxiv.org/abs/2005.05951}, 
}

@inproceedings{Renyi1961measures,
author={A. R'enyi},
title={On measures of entropy and information},
booktitle={Proceedings of the Fourth Berkeley Symposium on Mathematical Statistics and Probability, Volume 1: Contributions to the Theory of Statistics},
volume={4},
pages={547--562},
year={1961},
publisher={University of California Press}
}

@misc{sukhija2025maxinforlboostingexplorationreinforcement,
      title={MaxInfoRL: Boosting exploration in reinforcement learning through information gain maximization}, 
      author={Bhavya Sukhija and Stelian Coros and Andreas Krause and Pieter Abbeel and Carmelo Sferrazza},
      year={2025},
      eprint={2412.12098},
      archivePrefix={arXiv},
      primaryClass={cs.LG},
      url={https://arxiv.org/abs/2412.12098}, 
}

@misc{yu2021metaworldbenchmarkevaluationmultitask,
      title={Meta-World: A Benchmark and Evaluation for Multi-Task and Meta Reinforcement Learning}, 
      author={Tianhe Yu and Deirdre Quillen and Zhanpeng He and Ryan Julian and Avnish Narayan and Hayden Shively and Adithya Bellathur and Karol Hausman and Chelsea Finn and Sergey Levine},
      year={2021},
      eprint={1910.10897},
      archivePrefix={arXiv},
      primaryClass={cs.LG},
      url={https://arxiv.org/abs/1910.10897}, 
}

@misc{feng2023finetuningofflineworldmodels,
      title={Finetuning Offline World Models in the Real World}, 
      author={Yunhai Feng and Nicklas Hansen and Ziyan Xiong and Chandramouli Rajagopalan and Xiaolong Wang},
      year={2023},
      eprint={2310.16029},
      archivePrefix={arXiv},
      primaryClass={cs.LG},
      url={https://arxiv.org/abs/2310.16029}, 
}

@misc{oquab2024dinov2learningrobustvisual,
      title={DINOv2: Learning Robust Visual Features without Supervision}, 
      author={Maxime Oquab and Timothée Darcet and Théo Moutakanni and Huy Vo and Marc Szafraniec and Vasil Khalidov and Pierre Fernandez and Daniel Haziza and Francisco Massa and Alaaeldin El-Nouby and Mahmoud Assran and Nicolas Ballas and Wojciech Galuba and Russell Howes and Po-Yao Huang and Shang-Wen Li and Ishan Misra and Michael Rabbat and Vasu Sharma and Gabriel Synnaeve and Hu Xu and Hervé Jegou and Julien Mairal and Patrick Labatut and Armand Joulin and Piotr Bojanowski},
      year={2024},
      eprint={2304.07193},
      archivePrefix={arXiv},
      primaryClass={cs.CV},
      url={https://arxiv.org/abs/2304.07193}, 
}

@inproceedings{
liu2022metarewardnet,
title={Meta-Reward-Net: Implicitly Differentiable Reward Learning for Preference-based Reinforcement Learning},
author={Runze Liu and Fengshuo Bai and Yali Du and Yaodong Yang},
booktitle={Advances in Neural Information Processing Systems},
editor={Alice H. Oh and Alekh Agarwal and Danielle Belgrave and Kyunghyun Cho},
year={2022},
url={https://openreview.net/forum?id=OZKBReUF-wX}
}

@misc{agarwal2022deepreinforcementlearningedge,
      title={Deep Reinforcement Learning at the Edge of the Statistical Precipice}, 
      author={Rishabh Agarwal and Max Schwarzer and Pablo Samuel Castro and Aaron Courville and Marc G. Bellemare},
      year={2022},
      eprint={2108.13264},
      archivePrefix={arXiv},
      primaryClass={cs.LG},
      url={https://arxiv.org/abs/2108.13264}, 
}
\newpage

\appendix
\section{Algorithms}
{\setlength{\textfloatsep}{4pt}
 \setlength{\intextsep}{4pt}
\begin{algorithm}[h]
\caption{Optimistic Preference Pair Selection}
\label{alg:optimistic_pref}
\begin{algorithmic}
\STATE {\bfseries Input:} replay buffer $\mathcal{B}$, reward ensemble $\{r^{(m)}_\theta\}_{m=1}^{E}$, pref.\ buffer $\mathcal{P}$
\STATE {\bfseries Hyperparameters:} segment length $L$, pairs to add $K$
\STATE Candidate set $\mathcal{C}\leftarrow\emptyset$
\FOR{each batch}
  \STATE Sample candidate segment pairs $(\tau_1,\tau_2)$ (length $L$) from $\mathcal{B}$
  \FOR{each $(\tau_1,\tau_2)$}
    \STATE $R_i^{(m)} \leftarrow \sum\limits_{t=0}^{L-1} r^{(m)}_\theta(s_t^{i},a_t^{i}),\;
           i\!\in\!\{1,2\},\; m\!\in\![E]$
    \STATE $p^{(m)} \leftarrow \mathrm{softmax}([R_1^{(m)},R_2^{(m)}])_1$
    \STATE $\mathrm{score} \leftarrow \mu(p^{(1:E)}) + \sigma(p^{(1:E)})$
    \STATE $\mathcal{C}\leftarrow \mathcal{C}\cup\{(\tau_1,\tau_2,\mathrm{score})\}$
  \ENDFOR
\ENDFOR
\STATE Add top-$K$ pairs in $\mathcal{C}$ by score to $\mathcal{P}$
\end{algorithmic}
\end{algorithm}}
\algshort\ selects the most informative preference pairs for labeling by scoring each candidate segment pair using both predicted reward and reward-model epistemic uncertainty, then globally ranking and selecting the top-$K$ pairs to add to the preference buffer $\mathcal{P}$.

\newpage
\section{Implementation Details}
\subsection{Hyperparameters}

UBP2 learns a reward model from preference feedback using short trajectory segments. Preferences are periodically collected and the reward model is updated in batches.

\begin{table}[h]
\centering
\caption{Preference Learning Hyperparameters}
\resizebox{\textwidth}{!}{
\begin{tabular}{lcccccccccc}
\hline
\textbf{Parameter} & \textbf{Door Close} & \textbf{Window Close} & \textbf{Handle Press} & \textbf{Coffee Button} & \textbf{Faucet Open} & \textbf{Door Open} & \textbf{Door Unlock} & \textbf{Sweep Into} & \textbf{Drawer Open} & \textbf{Hammer}  \\
\hline
number of seeds & 10 & 10 & 10 & 10 & 15 & 15 & 10 & 10 & 10 & 15 \\
pref\_segment\_length & 10 & 10 & 10 & 10 & 10 & 10 & 10 & 10 & 10 & 10 \\
reward\_train\_horizon & 10 & 10 & 10 & 10 & 10 & 10 & 10 & 10 & 10 & 10 \\
num\_pref\_per\_addition & 12 & 12 & 12 & 25 & 12 & 12 & 12 & 25 & 25 & 50 \\
max\_pref\_feedback & 500 & 500 & 1000 & 1000 & 2000 & 2000 & 2500 & 5000 & 5000 & 10000 \\
pref\_update\_freq & 2500 & 2500 & 2500 & 2500 & 2500 & 2500 & 2500 & 2500 & 2500 & 2500 \\
reward\_batch\_size & 50 & 50 & 50 & 50 & 50 & 50 & 50 & 100 & 100 & 200 \\
\hline
\end{tabular}
}
\label{tab:preference_learning}
\end{table}

We use Model Predictive Control (MPC) with the Cross-Entropy Method (CEM) for planning. At each timestep, candidate trajectories are sampled and iteratively refined using elite selection
\begin{table}[h]
\centering
\caption{MPC (CEM) Planning Hyperparameters}
\resizebox{\textwidth}{!}{
\begin{tabular}{lcccccccccc}
\hline
\textbf{Parameter}& \textbf{Door Close} & \textbf{Window Close} & \textbf{Handle Press} & \textbf{Coffee Button} & \textbf{Faucet Open} & \textbf{Door Open} & \textbf{Door Unlock} & \textbf{Sweep Into} & \textbf{Drawer Open} & \textbf{Hammer} \\
\hline
iterations & 6 & 6 & 6 & 6 & 6 & 6 & 6 & 6 & 6 & 6 \\
number of samples & 512 & 512 & 512 & 512 & 512 & 512 & 512 & 512 & 512 & 512 \\
num of elites & 64 & 64 & 64 & 64 & 64 & 64 & 64 & 64 & 64 & 64 \\
num of policy trajectories & 24 & 24 & 24 & 24 & 24 & 24 & 24 & 24 & 24 & 24 \\
horizon & 7 & 7 & 7 & 11 & 7 & 11 & 11 & 7 & 7 & 7 \\
\hline
\end{tabular}
}
\label{tab:mpc_cem}
\end{table}

\subsection{Compute Resources}

All experiments were conducted in a shared high-performance computing cluster. We used GPU-accelerated nodes for training and evaluation. Specifically, we utilized A100 or H100 GPU partitions depending on availability. ~\Cref{tab:compute_resources} summarizes the hardware used in our experiments.

\begin{table}[h]
\centering
\caption{Compute resources used in all experiments.}
\begin{tabular}{llc}
\hline
\textbf{Resource} & \textbf{Specification} & \textbf{Quantity} \\
\hline
GPU (A100) & 5GB MIG partition & 1 \\
GPU (H100) & 10GB MIG partition (HBM3) & 1 \\
\hline
\end{tabular}
\label{tab:compute_resources}
\end{table}

\newpage
\section{Additional Experiments}
\subsection{Visual Results}

We evaluate the generalization of preference-based reinforcement learning to high-dimensional visual observations and using DinoV2 encoding \cite{oquab2024dinov2learningrobustvisual} of the observations. 
Experiments are conducted on 2 standard continuous-control benchmarks: \emph{Walker Walk},  and \emph{Cheetah Run}. On Walker Walk, both methods UBP2 and MBP achieve comparable performance and outperform model free baselines. In contrast,in  Cheetah Run, model free baseline outperform model based approach and UBP2. However, UBP2 still outperforms MBP. These results suggest that optimism helps mitigate the additional uncertainty associated with high-dimensional visual observations in model based approaches.
\begin{figure*}[h]
\centering
\makebox[\textwidth]{%
\hfill
\subfigure[Visual: Walker Walk (5000)]{%
    \begin{minipage}{0.235\textwidth}
    \centering
        \includegraphics[width=\textwidth]{paper_figures/walker-walk_IQM_MERGED.pdf}
    \end{minipage}
}%
\hfill
\subfigure[Visual: Cheetah Run (5000)]{%
    \begin{minipage}{0.235\textwidth}
    \centering
        \includegraphics[width=\textwidth]{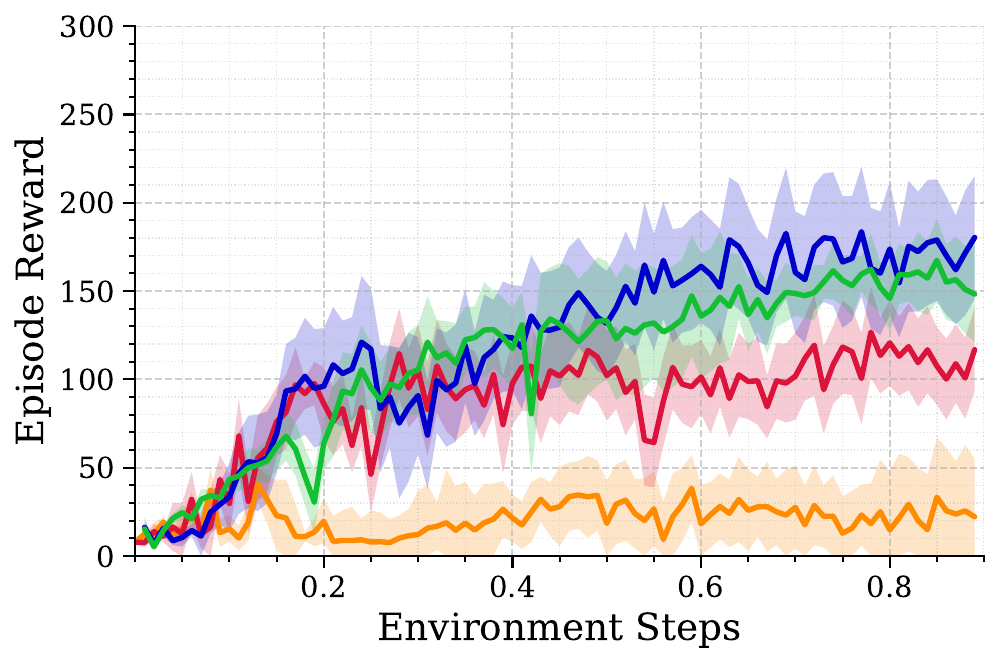}
    \end{minipage}
}%
\hfill
}
% ---------------- Legend ----------------
\vspace{-10pt}
\includegraphics[width=0.6\textwidth]{paper_figures/legend.pdf}
\vspace{-5pt}
\caption{
Performance comparison across visual DMControl tasks.
Curves show IQM performance with shaded standard error over 5 seeds.
Numbers between brackets indicate the maximum preference feedback budget.}
\label{fig:vis_main_results}
\end{figure*}

\subsection{Conservative Evaluation}
We additionally evaluate whether the uncertainty estimates learned by UBP2 can be used conservatively at test time by reversing the optimism term during evaluation, i.e., penalizing trajectories with high epistemic uncertainty instead of rewarding them. This experiment is motivated by the observation that optimism is beneficial during data collection and learning, where exploration is required, but may become undesirable during deployment if a risk-averse or safety-oriented policy is preferred. Importantly, optimism is not used during standard evaluation in UBP2; the learned policy is evaluated without any exploration bonus. Here, we instead explicitly introduce a conservative evaluation objective by multiplying the uncertainty bonus by $-1$ during planning.
Results shows that conservative evaluation substantially degrades performance on most tasks We speculate that uncertainty in UBP2 primarily acts as a useful exploratory signal rather than a reliable proxy for risk, and that naively enforcing conservatism at test time can significantly reduce task performance.
\begin{table*}[tbh]
\centering
\caption{Performance comparison across MetaWorld tasks for UBP2 with conservative evaluation. Entries report average IQM success with standard error at the final evaluation step. Best results are shown in bold}
\resizebox{0.7\textwidth}{!}{
\begin{tabular}{cccccccc}
\toprule
\textbf{Method} & \textbf{Coffee Button} & \textbf{Faucet Open}  & \textbf{Sweep Into} & \textbf{Drawer Open}  \\
\midrule
UBP2 & \textbf{86.2\,({\footnotesize 3.7})} & \textbf{89.5\,({\footnotesize 7.9})} & \textbf{36.4\,({\footnotesize 12.3})} & \textbf{77.7\,({\footnotesize 4.9})} \\
\midrule
UBP2 w conservative eval &  {51.0\,({\footnotesize 22.3})} & {85.9\,({\footnotesize 20.0})}  & {6.4\,({\footnotesize 0.0})} & {30.8\,({\footnotesize 23.7})} \\
\bottomrule
\end{tabular}
}
\end{table*}

\subsection{Varying Horizon}
\label{subsec:vary_horizon}
In the ablation study, we compare planning horizons of 7 and 11 across tasks. While horizon 11 yields better performance on a subset of tasks reported in the experiments section, it does not consistently improve results on the tasks considered in this appendix. In contrast, a horizon of 7 provides more stable and higher overall performance for these tasks. Therefore, we adopt horizon 7 for the results reported in this section.
\newpage

\begin{table*}[tbh]
\centering
\caption{Effect of planning horizon on UBP2 performance across a subset of MetaWorld tasks over 1M steps. Entries report IQM success with standard error at the final evaluation step.}
\resizebox{0.7\textwidth}{!}{
\begin{tabular}{ccccccc}
\toprule
{\textbf{Horizon}} & \textbf{Door Close} & \textbf{Window Close} & \textbf{Handle Press} & \textbf{Sweep Into} & \textbf{Drawer Open} \\
\midrule

7  & \textbf{97.8\,({\footnotesize 1.0})} & \textbf{93.1\,({\footnotesize 1.0})} & \textbf{93.4\,({\footnotesize 0.0})} & \textbf{36.4\,({\footnotesize 12.3})} & \textbf{77.7\,({\footnotesize 4.9})} \\

11 & 96.3\,({\footnotesize 0.0})
   & 91.6\,({\footnotesize 0.0})
   & 86.1\,({\footnotesize 8.0})
   & 15.0\,({\footnotesize 7.5})
   & 71.5\,({\footnotesize 19.6}) \\
\bottomrule
\end{tabular}
}
\label{tab:vary_horizon}
\end{table*}

\section{Proofs}
\subsection{Definitions and Clarifications}\label{sec:definitions}
We first give some definitions and characterizations of our learned models. Later, we clarify some of the choices we have made in regard to the theory. 

\paragraph{Closed forms of the mean and variance of GPs}
Let $K_n\in \mathbb{R}^{M\times M}$ denote the kernel matrix over $\mathcal{D}_n$, $k_{n}(s,a)\in\mathbb{R}^M$ denote the vector of kernel evaluations between $(s,a)$ and all inputs in $\mathcal{D}_{n}$, and $s'^j_{1:M}$ denote the noisy measurements of $p^*_j$. For GP regression of the true dynamics with observation noise variance $\sigma^2$, the posterior mean and variance for coordinate $j$ have the following standard closed forms \cite{sukhija2025sombrlscalableoptimisticmodelbased}:
\begin{align}
    \mu^{d}_{n,j}(s,a) 
    &= 
    k_{n}(s,a)^\top \bigl(K_{n}+\sigma^2 I\bigr)^{-1} s'^j_{1:M},
    \\
    \bigl(\sigma^{d}_{\text{epi}_{n,j}}(s,a)\bigr)^2
    &=
    k((s,a),(s,a))-
    k_{n}(s,a)^\top \bigl(K_{n}+\sigma^2 I\bigr)^{-1} k_{n}(s,a).
    \label{eq:gp-posterior}
\end{align}

We collect these coordinate-wise values into vector estimates $\mu^d_{n}(s,a)=[\mu^d_{n,j}(s,a)]_{j\le d_s}$ and 
$\sigma^d_{\text{epi}_n}(s,a)=[\sigma^d_{\text{epi}_{n,j}}(s,a)]_{j\le d_s}$. We similarly fit a GP model using \eqref{eq:gp-posterior} on the reward with posterior mean $\mu^r_{n}(s,a)$ and standard deviation $\sigma^r_{\text{epi}_n}(s,a)$ using the reward kernel $k_r$. Furthermore, we say that these learned models are \emph{well-calibrated} \cite{rothfuss2023hallucinatedadversarialcontrolconservative,sukhija2025sombrlscalableoptimisticmodelbased}. 

\paragraph{Definition of well-calibrated models} To connect epistemic uncertainty to optimism, in section \ref{sec:theory-results} we assumed the learned models are \emph{well-calibrated}:
their posterior standard deviations provide valid confidence radii around the true functions they are attempting to estimate. Concretely, we give the following definition. 
\begin{definition}[Well-calibrated learned dynamics and reward \cite{rothfuss2023hallucinatedadversarialcontrolconservative}] \label{def:well-calibrated} Simultaneously for all episodes $n$ and all $(s,a)\in\mathcal{S}\times\mathcal{A}$, with probability $1-\delta$:
\begin{align}
\bigl|r(s,a)-\mu^r_{n}(s,a)\bigr|
&\le 
\beta^r_n(\delta)\,\sigma^r_\mathrm{epi_{n}}(s,a) \label{eq:rew-calib},
\\
\forall j\le d_s:\qquad
\bigl|p^*_j(s,a)-\mu^d_n(s,a)\bigr|
&\le 
\beta^d_n(\delta)\,\sigma^d_\mathrm{epi_{n,j}}(s,a),
\label{eq:calib}
\end{align}
where $\beta^r_n(\delta),\beta^d_n(\delta)\in\mathbb{R}_{\geq0}$ are non-decreasing sequences. Moreover, in the GP/RKHS setting, typically
$\beta^r_n\in\mathcal{O}\!\big(\sqrt{\Gamma_{r,n}}\big)$ and
$\beta^d_n\in\mathcal{O}\!\big(\sqrt{\Gamma_{d,n}}\big)$ \cite{chowdhury2017kernelizedmultiarmedbandits}.

\end{definition}

We invoke a lemma in \cite{rothfuss2023hallucinatedadversarialcontrolconservative} to show that GP estimates of the dynamics and reward are well-calibrated models if Assumption \ref{ass:RKHS-regularity} holds:
\begin{lemma} [Well-calibrated confidence intervals for RKHS \cite{rothfuss2023hallucinatedadversarialcontrolconservative}] \label{lem:well-cal-model}Let $p^*$ be in a RKHS with kernel $k_d$ and the characteristics in Assumption \ref{ass:RKHS-regularity}. Suppose $\mu^d_n$ and $\sigma^{d}_{\mathrm{epi}_{n}}$ are the posterior mean and variance of a GP with kernel $k_d$ as defined in equation \eqref{eq:gp-posterior}. Then, there exists $\beta_n^d(\delta)$ for which the tuple $(\mu^d_n,\sigma^{d}_{\mathrm{epi}_{n}},\beta_n^d(\delta))$ is a well-calibrated model of $p^*$. The same is true for the proxy reward function $r(s,a)$.
\end{lemma}
\begin{proof}
Lemma \ref{lem:well-cal-model} is directly stated in \cite{rothfuss2023hallucinatedadversarialcontrolconservative} for an RKHS of vector-valued, scalar-output functions. Hence, this can directly be applied to the reward model and element-wise for the dynamics model.
\end{proof}

\paragraph{Definition of the maximum information gain} For a kernel $k$ and noise variance $\sigma^2$, the maximum information gain is defined to be:
\begin{equation}
\Gamma_M(k)
:=
\max_{A\subset \mathcal{S}\times\mathcal{A},\ |A|\le M}
\frac{1}{2}\log\det\!\left(I+\sigma^{-2}K_A\right),
\label{eq:info-gain}
\end{equation}
where $K_A$ is the kernel matrix over the set $A$.
Intuitively, $\Gamma_M(k)$ measures the maximum amount of information that $M$ noisy function evaluations can reveal
about a function in the RKHS class parameterized by $k$.
This quantity is sublinear in $M$ for many common kernels \cite{sukhija2025sombrlscalableoptimisticmodelbased}. See Table \ref{tab:max_info_gain_kernels} for a list of the asymptotic growth of $\Gamma_M$ for different kernels.

\begin{table}[!hbp]
\centering
\caption{Maximum information gain bounds for common choice of kernels \cite{sukhija2025sombrlscalableoptimisticmodelbased}.}
\label{tab:max_info_gain_kernels}
\begin{tabular}{ll}
\toprule
Kernel $k(s,s')$ & $\Gamma_N$ \\
\midrule
Linear: $s^\top s'$ 
& $\mathcal{O}\!\big(d\,\log N\big)$ \\[4pt]

RBF: $\exp\!\Big(-\dfrac{\|s-s'\|^2}{2l^2}\Big)$
& $\mathcal{O}\!\big(\log^{d+1} N\big)$ \\[6pt]

Mat\'ern:
$\dfrac{1}{\Gamma(\nu)\,2^{\nu-1}}
\Big(\dfrac{\sqrt{2\nu}\,\|s-s'\|}{l}\Big)^\nu
B_\nu\!\Big(\dfrac{\sqrt{2\nu}\,\|s-s'\|}{l}\Big)$
& $\mathcal{O}\!\Big(
N^{\frac{d}{2\nu+d}}\,
\log^{\frac{2\nu}{2\nu+d}}(N)
\Big)$ \\
\bottomrule
\end{tabular}
\end{table}

\paragraph{Further discussion on modeling the true reward $r^*$ in preference-based RL.} As discussed in section \ref{sec:pbrl} and the Limitations, in PbRL, the agent never observes the real scalar rewards, and is instead trained on preference labels, giving rise to the definition of the proxy reward function $r(s,a)$ within the same equivalence class of $r^*(s,a)$. This ``latent-reward" view is standard in the PbRL literature, where preferences provide information about an underlying reward signal and hence allows the learning of a reward model whose ordering of information matches the ground-truth ordering \cite{wirth2017survey, christiano2017pbrl}. We assume that the well-calibrated GP reward model effectively approximates the proxy reward function $r(s,a)$, satisfying equation \ref{eq:rew-calib}. This view has enabled us to effectively produce sublinear regret bounds for the problem setting presented in this paper.

However, as we have acknowledged earlier, it may be possible to make the reward learning error stemming from PbRL more explicit. Namely, \cite{zhan2024provablerewardagnosticpreferencebasedreinforcement} establishes value approximation guarantees under a linear reward parameterization, where the preference learning error can be more clearly characterized and directly incorporated into the final bound. In contrast, our setting does not assume such a parametric structure on the reward. As a result, explicitly propagating preference-learning error into the regret bound would require introducing additional assumptions on the reward model and its convergence rate. Introducing assumptions like a linear reward model could incorporate reward convergence into Lemma \ref{lem:ucb-optimism} and Lemma \ref{lem:inf-ucb-optimism}, but this would complicate the analysis while introducing harsher assumptions. Our current result isolates planning and provides guarantees under general reward uncertainty, avoiding strong reward model conditions.

\paragraph{Clarification between implemented dynamics and reward estimators and theoretical assumptions} As introduced in the Limitations, practically the dynamics and reward are implemented via deep ensemble networks, while in the theoretical analysis they are assumed to be GPs. We acknowledge that this may be a strict assumption, and that our actual implementation may not fully reflect the theory. 

However, we note that in practice, deep ensembles are ubiquitous in model-based RL to quantify epistemic uncertainty or model disagreement. Furthermore, similar approximations have been adopted in prior work, such as \cite{sukhija2025sombrlscalableoptimisticmodelbased}, providing a practical starting point for theoretically characterizing performance. Our intent is not to claim that the implemented ensemble estimator is an exact GP posterior. Our proof requires two properties: 1) a high-probability confidence relationship between the reward/dynamics model error and their uncertainty radii, and 2) a cumulative uncertainty bound on these uncertainty radii (in our case, via GP information gain bounds). GP posterior standard deviations provide these quantities in closed form, which is why they are used in Lemma \ref{lem:inf-ucb-optimism} and Theorem \ref{theorem:inf-regret-bound}. 

Hence, both our implementation and theory aim to capture the same underlying notion of posterior-predictive uncertainty, but using different model classes and approximations. In the implementation, ensemble disagreement and JRD are used as sample-based approximations to the spread of the posterior predictive distribution of the reward and dynamics. We view the ensemble members as approximate samples from a posterior or posterior-like distribution over possible reward and dynamics models. Then, the ensemble mean/disagreement can be viewed as Monte Carlo estimates of the mean/variance of the approximate posterior predictive distribution, and the JRD is a divergence-based measure of disagreement among each ensemble member's predictive distributions. Therefore, both our implementation and theory use predictive uncertainty as a measurement of confidence, but only the GP case yields an explicit, clean, closed form regret bound.

We acknowledge, however, that finite-neural ensembles are not guaranteed to be as well-calibrated as GPs, potentially presenting challenges reconciling the derived regret bounds with the practical results.

\subsection{Analysis of Discounted Finite Horizon Planning with UCB Objective} \label{sec:lemma-proofs}

\paragraph{Finite-horizon objective and regret}Prior to showing the theoretical justification for the full infinite-horizon regret bound and the full UCB objective, we first consider a reduced UCB objective. Define:

\begin{equation}
% \begin{aligned}
\eta_{n}^{\mathrm{UCB}}(\pi,\mu_n^d)
:= \mathbb{E}_{\pi,\mu^d_{n}}
\Bigg[
\sum_{t=0}^{H-1}\gamma^t
\Big(
\mu^r_{n}(\hat s_t,a_t)
+ a_n \sigma^r_{n}(\hat s_t,a_t)+ b_n \sigma^d_{n}(\hat s_t,a_t)
\Big)\Bigg].\label{eq:theory-plan-objective-reduced}
% \end{aligned} 
\end{equation}

This objective corresponds to a finite-horizon variant of the full UCB objective in~\eqref{eq:theory-plan-objective}, with $a_n, b_n$ defined as in Lemma \ref{lem:inf-ucb-optimism}. Since the objective is finite-horizon, we omit the terminal value estimation and its uncertainty bonus. In this section, we derive a finite-horizon regret bound for policies obtained by optimizing~\eqref{eq:theory-plan-objective-reduced} and how that it is sublinear. This informs our infinite horizon regret bound for the full UCB planning objective. 

Analogous to equation \eqref{eq:inf-regret}, we introduce a performance metric for the finite-horizon setting. We define the \emph{cumulative finite-horizon regret} for a particular policy $\pi_n$ as:
\begin{equation} \label{eq:fin-regret}
    R_{N}=\sum_{n=1}^{N}r_{n}=\sum_{n=1}^{N}\eta(\pi^*,p^*)-\eta(\pi_n,p^*),
\end{equation}where $\eta(\cdot\;,\cdot)$ is the finite-horizon analogue of \eqref{eq:inf-return}
\begin{equation} \label{eq:fin-return}
    \eta(\pi,p)=\mathbb{E}_{\pi,p}\left[\sum_{t=0}^{H-1} \gamma^t r(s_t,a_t)\right]. 
\end{equation}

\paragraph{Discounted Finite-Horizon Cumulative Uncertainty and Return} 
We further define the following quantities. For real episode $n$ and policy $\pi$, define (as in \cite{sukhija2025sombrlscalableoptimisticmodelbased}):
\begin{align}
\widetilde{\Sigma}^r_n(\pi) &:= \mathbb{E}_{\pi,\mu^d_{n}}\Big[\sum_{t=0}^{H-1} \gamma^t\sigma^r_\mathrm{epi_{n}}(\hat s_t,a_t)\Big] \label{sigma_sum_def1},\\
\Sigma^r_n(\pi) &:= \mathbb{E}_{\pi,p^*}\Big[\sum_{t=0}^{H-1} \gamma^t\sigma^r_\mathrm{epi_{n}}(s_t,a_t)\Big] \label{sigma_sum_def2},\\
\widetilde{\Sigma}^d_n(\pi) &:= \mathbb{E}_{\pi,\mu^d_{n}}\Big[\sum_{t=0}^{H-1} \gamma^t\norm{\sigma^d_\mathrm{epi_{n,j}}(\hat s_t, a_t)}\Big] \label{sigma_sum_def3},\\
\Sigma^d_n(\pi) &:= \mathbb{E}_{\pi,p^*}\Big[\sum_{t=0}^{H-1} \gamma^t\norm{\sigma^d_\mathrm{epi_{n,j}}( s_t,a_t)}\Big]. \label{sigma_sum_def4}
\end{align}
These quantities represent the sum of the estimated uncertainties of the learned reward and dynamics across timesteps under policy rollouts in the true and mean dynamics. Definitions \eqref{sigma_sum_def1} to \eqref{sigma_sum_def4} will be useful later proofs, and will be generalized to the infinite horizon setting.

\begin{lemma} [Upper Bound of Discounted Return]\label{lem:return_bounds}
Consider the equations \eqref{eq:fin-return}, \eqref{sigma_sum_def3}, and \eqref{sigma_sum_def4} with the true dynamics $p^*$ and mean dynamics $\mu^d_n$ having the same process noise. Then, there is a $\lambda_n\in\mathcal{O}(\sqrt{\Gamma_{d,N}})$ such that for any policy $\pi$:

\begin{align}
|\eta(\pi,p^*)-\eta(\pi,\mu_{n}^d)|&\leq\lambda_n \Sigma_n^d(\pi). \label{eq:return_diff_true} \\
|\eta(\pi,p^*)-\eta(\pi,\mu_{n}^d)|&\leq\lambda_n\tilde\Sigma_n^d(\pi).
\end{align}
\end{lemma}
\begin{proof}
Lemma \ref{lem:return_bounds} is simply a discounted variation of Lemma B.1. from \cite{sukhija2025sombrlscalableoptimisticmodelbased}. It can be seen that adding the discount factor has no effect on the proof.
\end{proof}

\begin{lemma}[Optimism of the UCB planning objective]\label{lem:ucb-optimism} Let Assumptions \ref{ass:cont-dyn-and-r}, \ref{ass:RKHS-regularity} hold, and let the learned GP dynamics and reward models be well-calibrated. Then there exist coefficients $a_n\in\mathcal{O}\!\big(\sqrt{\Gamma_{r,N}}\big)$ and
$b_n\in\mathcal{O}\!\big(\sqrt{\Gamma_{d,N}}\big)$ such that
for all episodes $n$ and all policies $\pi$, with probability at least $1-\delta$:
\begin{equation}
\eta(\pi,p^*) \;\le\; \eta_n^{\mathrm{UCB}}(\pi,\mu_n^d).
\end{equation}
\end{lemma}

\begin{proof} From Lemma \ref{lem:return_bounds}, we have:

{\allowdisplaybreaks
\begin{align}
\eta(\pi,p^*) 
&\le \lambda_n \tilde\Sigma_n^d(\pi) + \eta(\pi,\mu_n^d)
\nonumber\\
&= \lambda_n \tilde\Sigma_n^d(\pi)
+ \mathbb{E}_{\pi,\mu_n^d}\Big[\sum_{t=0}^{H-1} \gamma^t r(\hat s_t,a_t)\Big]
\nonumber\\
&= \lambda_n \tilde\Sigma_n^d(\pi)
+\mathbb{E}_{\pi,\mu_n^d}\Big[\sum_{t=0}^{H-1} \gamma^t\big(r(\hat s_t,a_t)-\mu^r_n(\hat s_t,a_t)\big)\Big] + \mathbb{E}_{\pi,\mu_n^d}\Big[\sum_{t=0}^{H-1} \gamma^t \mu^r_n(\hat s_t,a_t)\Big]
\nonumber\\
&\le \lambda_n \tilde\Sigma_n^d(\pi)
+ \underbrace{\mathbb{E}_{\pi,\mu_n^d}\Big[\sum_{t=0}^{H-1} \gamma^t\big|r^*(\hat s_t,a_t)-\mu^r_n(\hat s_t,a_t)\big|\Big]}_{(\mathrm{A})}
+ \mathbb{E}_{\pi,\mu_n^d}\Big[\sum_{t=0}^{H-1} \gamma^t \mu^r_n(\hat s_t,a_t)\Big].
\nonumber
\end{align}}
We bound $(\mathrm{A})$ using Definition \ref{def:well-calibrated}, arriving at the optimism lemma:

\begin{align*}
\eta(\pi,p^*) &\leq \lambda_n \tilde\Sigma_n^d(\pi)
+ \beta^r_n\tilde\Sigma^r_{n}(\pi)
+ \mathbb{E}_{\pi,\mu_n^d}\Big[\sum_{t=0}^{H-1} \gamma^t \mu^r_n(\hat s_t,a_t)\Big] \\
&=\eta^{UCB}_n(\pi, \mu_n^d), &&\text{(by definition)}
\end{align*}
where $a_n,b_n$ correspond to the functions $\beta^r_n, \lambda_n$, respectively. Since $\lambda_n \in\mathcal{O}(\sqrt{\Gamma_{d,N}})$, and $\beta^r_n\in \mathcal{O}(\sqrt{\Gamma_{r,N}})$, we arrive at the theoretical bounds for coefficients $a_n,b_n$ presented in Lemma \ref{lem:ucb-optimism}. 
\end{proof}

Now, we introduce the analogue of Assumption \ref{ass:no-plan-subopt} for the reduced UCB objective. This will be used in the proof of Lemma \ref{lem:single_regret_bound} below.

\begin{assumption}[No planner sub-optimality for reduced UCB objective] \label{ass:no-plan-subopt-reduced}The policy $\pi^{\mathrm{red}}_n$ induced by the planning procedure through optimizing the reduced UCB objective \eqref{eq:theory-plan-objective-reduced} is equal to the optimal policy $\pi^{*,\mathrm{red}}_n$ which maximizes \eqref{eq:theory-plan-objective-reduced}. That is, $\pi^{\mathrm{red}}_n=\pi^{*,\mathrm{red}}_n$.
\end{assumption}

\begin{lemma}[Single episode regret bound]\label{lem:single_regret_bound}
Let Assumptions \ref{ass:cont-dyn-and-r}, \ref{ass:RKHS-regularity}, and \ref{ass:no-plan-subopt-reduced} hold, and assume that the learned dynamics and reward are well-calibrated. Consider the definition of $r_n$ in \eqref{eq:fin-regret}, and let $a_n,b_n$ be as defined in Lemma \ref{lem:ucb-optimism} . Then, $\forall n>0$ with probability at least $1-\delta$,
\begin{equation}
r_n\leq(b'^2_n+2b'_n)\Sigma^d_n(\pi^{\mathrm{red}}_n)+(a_n^2+a_n)\Sigma^r_n(\pi^{\mathrm{red}}_n),\label{eq:single_regret_bound}
\end{equation}
where $b'_n=\max\{b_n,\lambda_n\}$ and $\pi^{\mathrm{red}}_n$ is the planner induced policy, assumed to maximize equation \eqref{eq:theory-plan-objective-reduced}.
\end{lemma}
\begin{proof}
We start with the definition of $r_n$ presented in \eqref{eq:fin-regret}:

\begin{align*}
r_n&=\eta(\pi^*,p^*)-\eta(\pi_n,p^*) \\
&\leq \eta^{UCB}_n(\pi^*,\mu_n^d)-\eta(\pi^{\mathrm{red}}_n,p^*) &&\text{(Lemma \ref{lem:ucb-optimism})} \\
&\leq \eta^{UCB}_n(\pi^{*,\mathrm{red}}_n,\mu_n^d)-\eta(\pi^{\mathrm{red}}_n,p^*)&&\text{(Optimization of \eqref{eq:theory-plan-objective-reduced})}\\
&=\eta^{UCB}_n(\pi^{\mathrm{red}}_n,\mu_n^d)-\eta(\pi^{\mathrm{red}}_n,p^*)&&\text{(Assumption \ref{ass:no-plan-subopt-reduced})}\\
&=\eta(\pi^{\mathrm{red}}_n,\mu^d_n)+a_n\tilde\Sigma^r_n(\pi^{\mathrm{red}}_n)+b_n\tilde\Sigma^d_n(\pi^{\mathrm{red}}_n)-\eta(\pi^{\mathrm{red}}_n,p^*) \\
&\leq \lambda_n\Sigma^d_n(\pi^{\mathrm{red}}_n)+a_n\tilde\Sigma^r_n(\pi^{\mathrm{red}}_n)+b_n\tilde\Sigma^d_n(\pi^{\mathrm{red}}_n) &&\text{(Lemma \ref{lem:return_bounds})}
\end{align*}

Let $b'_n=\max\{b_n,\lambda_n\}$. Then

\begin{align*}
r_n&\leq2b'_n\Sigma^d_n(\pi^{\mathrm{red}}_n)+b'_n(\tilde\Sigma^d_n(\pi^{\mathrm{red}}_n)-\Sigma^d_n(\pi^{\mathrm{red}}_n))+a_n\Sigma^r_n(\pi^{\mathrm{red}}_n)+a_n(\tilde\Sigma^r_n(\pi^{\mathrm{red}}_n)-\Sigma^r_n(\pi^{\mathrm{red}}_n)) \\
&\leq (2b'_n+b'^2_n)\Sigma^d_n(\pi^{\mathrm{red}}_n)+(a_n^2+a_n)\Sigma^r_n(\pi^{\mathrm{red}}_n).
\end{align*}
In the last inequality, we use the fact that $\sigma^r_n(\cdot)$ and $\norm{\sigma^d_n(\cdot)}$ are bounded and positive. Hence, they qualify as valid ``reward signals" which can be substituted into Lemma \ref{lem:return_bounds}, yielding $\tilde\Sigma^d_n(\pi^{\mathrm{red}}_n)-\Sigma^d_n(\pi^{\mathrm{red}}_n)
\;\le\;
b'_n\Sigma^d_n(\pi^{\mathrm{red}}_n)$ and $\tilde\Sigma^r_n(\pi^{\mathrm{red}}_n)-\Sigma^r_n(\pi^{\mathrm{red}}_n)
\;\le\;
a_n\Sigma^r_n(\pi^{\mathrm{red}}_n)$ \cite{sukhija2025sombrlscalableoptimisticmodelbased}
\end{proof}

\begin{theorem} [Regret bound with dynamics and reward uncertainty bonuses]\label{theorem:main-regret-bound} Let Assumptions \ref{ass:cont-dyn-and-r}, \ref{ass:RKHS-regularity}, and \ref{ass:no-plan-subopt} hold, and let the learned GP dynamics and reward models be well-calibrated. The cumulative regret after $N$ episodes of real environment interaction satisfies, with probability $1-\delta$:
\begin{equation}
R_N
\;\le\;
\mathcal{O}\!\Big(
H^3 \sqrt{N}\,
\big(\Gamma_{d,N}^{3/2} + \Gamma_{r,N}^{3/2}\big)
\Big).
\label{eq:main-regret}
\end{equation}
\end{theorem}

\begin{proof}
We begin with the definition of the cumulative regret $R_N$:
\begin{align*}
R_N 
&= \sum_{n=1}^N r_n \\
&\le \sum_{n=1}^{N}\Big[(2b'_n+b'^2_n)\Sigma^d_n(\pi^{\text{ind}}_n)+(a_n^2+a_n)\Sigma^r_n(\pi^{\mathrm{red}}_n)\Big]
\\
&\le (2b'_N+b'^2_N)\sum_{n=1}^{N}\Sigma^d_n(\pi^{\mathrm{red}}_n)
\;+\;
(a_N^2+a_N)\sum_{n=1}^{N}\Sigma^r_n(\pi^{\mathrm{red}}_n) \\
&\le (2b'_N+b'^2_N)\sum_{n=1}^{N}\mathbb{E}_{p^*}\Big[\sum_{t=0}^{H-1}\big\|\sigma_{\mathrm{epi}_{n}}^d(s_t,\pi^{\mathrm{red}}_n(s_t))\big\|\Big]
\;+\;
(a_N^2+a_N)\sum_{n=1}^{N}\mathbb{E}_{p^*}\Big[\sum_{t=0}^{H-1}\sigma_{\mathrm{epi}_{n}}^r(s_t,\pi^{\mathrm{red}}_n(s_t))\Big] \\
&\le
(2b'_N+b'^2_N)\,\sqrt{NH}\,
\sqrt{
\sum_{n=1}^{N}\mathbb{E}_{p^*}\Big[\sum_{t=0}^{H-1}\big\|\sigma_{\mathrm{epi}_{n}}^d(s_t,\pi^{\mathrm{red}}_n(s_t))\big\|^2\Big]
}
\nonumber\\
&\hspace{3.6em}
+\;
(a_N^2+a_N)\,\sqrt{NH}\,
\sqrt{
\sum_{n=1}^{N}\mathbb{E}_{p^*}\Big[\sum_{t=0}^{H-1}\big(\sigma_{\mathrm{epi}_{n}}^r(s_t,\pi^{\mathrm{red}}_n(s_t))\big)^2\Big]
} \quad\text{(Cauchy-Schwartz + Jensen's inequality)}\\
&\leq C_1(2b'_N+b'^2_N)H\sqrt{N\Gamma_{d,NH}}+C_2(a_N^2+a_N)H\sqrt{N\Gamma_{r,NH}} \quad \text{(From Lemma 17 in \cite{curi2020efficientmodelbasedreinforcementlearning})}
\end{align*}
Finally, notice that from Lemma \ref{lem:ucb-optimism} and Lemma \ref{lem:return_bounds} that $a_N\propto H\beta^r_n$, $b_N \propto \lambda_N$, $\lambda_N\propto H\beta^d_n$, $\beta_n^d\in\mathcal{O} (\sqrt{\Gamma_{d,N}})$, and $\beta_n^r\in \mathcal{O}(\sqrt{\Gamma_{r,N}})$. Substituting, we arrive at \eqref{eq:main-regret}. 
\end{proof} 

\newpage
\subsection{Analysis of Discounted Infinite Horizon Planning with the Complete UCB Objective } \label{sec:inf-bound-proof}

We now utilize and generalize the ideas for the finite horizon regret proof in Appendix \ref{sec:lemma-proofs} to derive our final infinite-horizon regret bound with the complete UCB objective \eqref{eq:theory-plan-objective}. 

We start by extending the definitions of the finite-horizon cumulative uncertainty (\ref{sigma_sum_def1})-(\ref{sigma_sum_def4}) to be infinite-horizon.  

\textbf{Discounted Infinite-Horizon Cumulative Uncertainty} 
For episode $n$ and policy $\pi$, define

\begin{align}
\widetilde{\Sigma}^r_{\gamma,n}(\pi) &:= \mathbb{E}_{\pi,\mu^d_{n}}\Big[\sum_{t=0}^{\infty} \gamma^t\sigma^r_\mathrm{epi_{n}}(\hat s_t,a_t)\Big] \label{inf_sigma_sum_def1},\\
\Sigma^r_{\gamma,n}(\pi) &:= \mathbb{E}_{\pi,p^*}\Big[\sum_{t=0}^{\infty} \gamma^t\sigma^r_\mathrm{epi_{n}}(s_t,a_t)\Big] \label{inf_sigma_sum_def2},\\
\widetilde{\Sigma}^d_{\gamma,n}(\pi) &:= \mathbb{E}_{\pi,\mu^d_{n}}\Big[\sum_{t=0}^{\infty} \gamma^t\norm{\sigma^d_\mathrm{epi_{n,j}}(\hat s_t, a_t)}\Big] \label{inf_sigma_sum_def3},\\
\Sigma^d_{\gamma,n}(\pi) &:= \mathbb{E}_{\pi,p^*}\Big[\sum_{t=0}^{\infty} \gamma^t\norm{\sigma^d_\mathrm{epi_{n,j}}( s_t,a_t)}\Big]. \label{inf_sigma_sum_def4}
\end{align}

By Assumption \ref{ass:q-uncertainty}, there exists constant $\alpha_d$ and $\alpha_r$ such that for all episode $n$,
    $$\sigma_n^q(s,a)=\alpha_d\mathbb{E}\Big[\sum_{n=1}^\infty\gamma^t\lVert\sigma_{\mathrm{epi}_n}^d(s_t,a_t)\rVert\Big]+\alpha_r\mathbb{E}\Big[\sum_{n=1}^\infty\gamma^t\sigma_{\mathrm{epi}_n}^r(s_t,a_t)\Big]. $$

We give a brief justification for this assumption by proving a lemma bounding the epistemic uncertainty of $Q$. We assume that the underlying value function $V$ is bounded by $V_{\max}$ and Lipchitz with a Lipchitz constant $L_V$. We start by the following lemma:

\begin{lemma}  \label{lem:q-estimation-bound}
    Let Assumption \ref{ass:cont-dyn-and-r}, \ref{ass:RKHS-regularity} hold. Given dynamics $p^*$, $\mu_n^d$ and rewards $r$, $\mu_n^r$. We have the following bound

    \begin{equation}
        |\bar{Q}_{\mu_n}^\pi(s,a)-Q^\pi_{p^*}(s,a)|\leq \beta_n^r\mathbb{E}_{\pi,\mu_n^d}\left[\sum_{t=0}^\infty \gamma^t \sigma_n^r(s_t,a_t)\right]+ \gamma L_{V_{p^*}^\pi}\beta_n^d\mathbb E_{\mu_n^d,\pi}\left[\sum_{t=0}^\infty\gamma^t\lVert\sigma_n^d(s_t,a_t)\rVert\right].
    \end{equation}

\begin{proof}
    Define the transition operators
    \begin{equation}
        (P^\pi_p f)(s,a) := \mathbb E_{s'\sim p(\cdot\mid s,a)}\big[f(s',\pi(s'))\big],
    \end{equation}
    \begin{equation}
        (P_p g)(s,a) := \mathbb E_{s'\sim p(\cdot\mid s,a)}\big[g(s')\big].
    \end{equation}

    Then we have,
         \begin{align*}
        \bar Q^\pi_{\mu_n^d}-Q^\pi_{p^*}
        &= (\mu_n^r-r) + \gamma\big(P^\pi_{\mu_n^d} \bar Q^\pi_{\mu_n^d}- P^\pi_{p^*} Q^\pi_{p^*}\big)\\
        &= (\mu_n^r-r) + \gamma P^\pi_{\mu_n^d}(\bar Q^\pi_{\mu_n^d}-Q^\pi_{p^*})
           + \gamma(P^\pi_{\mu_n^d}-P^\pi_{p^*})Q^\pi_{p^*}.
        \end{align*}

        Rearranging gives the resolvent form
        \begin{equation*}
            (I-\gamma P^\pi_{\mu_n^d})(\bar Q^\pi_{\mu_n^d}-Q^\pi_{p^*})
        = (\mu_n^r- r) + \gamma(P^\pi_{\mu_n^d}-P^\pi_{p^*})Q^\pi_{p^*}.
        \end{equation*}

        Since rewards are bounded, $(I-\gamma P^\pi_{\mu_n^d})^{-1}$ exists on bounded functions,
        use the Neumann series $(I-\gamma P^\pi_{\mu_n^d})^{-1}=\sum_{t=0}^\infty (\gamma P^\pi_{\mu_n^d})^t$ to get
        \begin{equation*}
            \bar Q^\pi_{\mu_n^d}-Q^\pi_{p^*}
            =
            \sum_{t=0}^\infty (\gamma P^\pi_{\mu_n^d})^t (\mu_n^r-r)
            \;+\;
            \gamma\sum_{t=0}^\infty (\gamma P^\pi_{\mu_n^d})^t (P^\pi_{\mu_n^d}-P^\pi_{p^*})Q^\pi_{p^*}.
        \end{equation*}

        Take absolute values pointwise and apply the triangle inequality:
        \begin{align*}
            \big|\bar Q^\pi_{\mu_n^d}-Q^\pi_{p^*}\big|
            &\le \sum_{t=0}^\infty \gamma^t \big|(P^\pi_{\mu_n^d})^t(\mu_n^r-r)\big| +\gamma \sum_{t=0}^\infty \gamma^t \big|(P^\pi_{\mu_n^d})^t (P^\pi_{\mu_n^d}-P^\pi_{p^*})Q^\pi_{p^*}\big|.
        \end{align*}
        Now interpret $(P^\pi_{\mu_n^d})^t$ as expectation over a length-$t$ rollout under $(\mu_n^d,\pi)$ starting from $(s,a)$:
        for any bounded $g$,
        $$
        ((P^\pi_{\mu_n^d})^t g)(s,a) = \mathbb E_{\mu_n^d,\pi}\!\left[g(s_t,a_t)\,\middle|\,s_0=s,a_0=a\right].
        $$
        Applying this with $g=\mu_n^r-r$ yields

        \begin{align*}
            \big|((P^\pi_{\mu_n^d})^t(\mu_n^r-r))(s,a)\big|
        &\le
        \mathbb E_{\mu_n^d,\pi}\!\left[\big|\mu_n^r(s_t,a_t)-r(s_t,a_t)\big|\,\middle|\,s_0=s,a_0=a\right] \\
        &\le \beta_n^r \mathbb{E}_{\pi,\mu_n^d}[\sigma_n^r(s,_t,a_t)]
        \end{align*}

        For the second term, note that
        \[
        \big((P^\pi_{\mu_n^d}-P^\pi_{p^*})Q^\pi_{p^*}\big)(s,a)
        =
        \mathbb E_{\mu_n^d}\!\big[V^\pi_{p^*}(s')\big]
        -
        \mathbb E_{p^*}\!\big[V^\pi_{p^*}(s')\big]
        =
        \big((P_{\mu_n^d}- P_{p^*})V^\pi_{p^*}\big)(s,a),
        \]
        so
        \begin{align*}
            \big| ((P^\pi_{\mu_n^d})^t (P^\pi_{\mu_n^d}- P^\pi_{p^*})Q^\pi_{p^*})(s,a)\big| &\le \mathbb E_{\mu_n^d,\pi}\!\left[\big|(P_{\mu_n^d}-P_{p^*})V^\pi_{p^*}(s_t,a_t)\big|\,\middle|\,s_0=s,a_0=a\right] \\
        &\le L_{V^\pi_{p^*}} \mathbb E_{\mu_n^d,\pi}[W(p_{\mu_n^d}(\cdot|s_t,a_t),p_{p^*}(\cdot|s_t,a_t))]\\
        &= L_{V^\pi_{p^*}} \mathbb E_{\mu_n^d,\pi}[|\mu_n^d(s_t,a_t)-p^*(s_t,a_t)|]\\
        &\le L_{V^\pi_{p^*}}\beta_n^d\mathbb E_{\mu_n^d,\pi}[\lVert\sigma_n^d(s_t,a_t)\rVert].
        \end{align*}

    $W(\cdot,\cdot)$ is the Wasserstein distance between two distributions. For two Gaussian distributions, this reduces to the distance between their means. 
\end{proof}
\end{lemma}

The true $Q$-error $|Q_{p^*}^\pi-\bar Q_{\mu_n}^\pi|$ is unobservable. However, Lemma \ref{lem:q-estimation-bound} shows that this error can be upper-bounded by a discounted accumulation of the posterior standard deviations of the dynamics and reward models. Therefore, we utilize this bound as a \emph{computable epistemic uncertainty proxy} for the $Q$-function. Intuitively, because $Q$ is a functional of the reward and dynamics, it's within reason that the epistemic uncertainty of $Q$ should be built from uncertainty the reward model and dynamics model. 

Note that although we explicitly reference Assumption \ref{ass:cont-dyn-and-r}, the Gaussian nature of the dynamics and reward in order to characterize their Wasserstein distance, this does not mean that the process noise also has to be Gaussian. Please refer to Appendix \ref{sec:sub-gaussian-noise} for further discussions on settings with sub-gaussian noise.

We can give a proof of Lemma \ref{lem:inf-ucb-optimism} following the previous proof of Lemma \ref{lem:ucb-optimism} and also treat $Q$ as the tail return estimator. 

Proof of Lemma \ref{lem:inf-ucb-optimism}. By Lemma B.7 in \cite{sukhija2025sombrlscalableoptimisticmodelbased},
{\allowdisplaybreaks
\begin{align*}
\eta_\gamma(\pi,p^*)
&\le \lambda_n \tilde\Sigma_{\gamma,n}^d(\pi) + \eta_\gamma(\pi,\mu_n^d)\\
&= \lambda_n \mathbb{E}_{\pi,\mu_n^d}\!\left[\sum_{t=0}^{\infty}\gamma^t \lVert \sigma_{\mathrm{epi}_n}^d(s_t,a_t) \rVert\right]
  + \mathbb{E}_{\pi,\mu_n^d}\!\left[\sum_{t=0}^{\infty}\gamma^t r(s_t,a_t)\right] \\
&= \lambda_n \mathbb{E}_{\pi,\mu_n^d}\!\left[\sum_{t=0}^{H-1}\gamma^t \lVert \sigma_{\mathrm{epi}_n}^d(s_t,a_t) \rVert\right]
  + \mathbb{E}_{\pi,\mu_n^d}\!\left[\sum_{t=0}^{H-1}\gamma^t (r(s_t,a_t)-\mu_n^r(s_t,a_t))\right] \\
&\quad + \mathbb{E}_{\pi,\mu_n^d}\!\left[\sum_{t=0}^{H-1}\gamma^t \mu_n^r(s_t,a_t)\right]
  + \gamma^H \mathbb{E}_{\pi,\mu_n^d}\!\left[\sum_{t=0}^{\infty}\gamma^t\big(\lVert\sigma_{\mathrm{epi}_n}^d(s_{t+H},a_{t+H})\rVert + r(s_{t+H},a_{t+H})\big)\right] \\ 
&\le \lambda_n \tilde\Sigma_{\gamma,n}^d(\pi)
  + \beta_n^r \tilde\Sigma_{\gamma,n}^r(\pi)
  + \mathbb{E}\!\left[\sum_{t=0}^{H-1}\gamma^t \mu_n^r(s_t,a_t)\right] \\
&\quad + \gamma^H \lambda_n \mathbb{E}\!\left[\sum_{t=0}^{\infty}\gamma^t \lVert\sigma_{\mathrm{epi}_n}^d(s_{t+H},a_{t+H})\rVert\right]
  + \gamma^H \beta_n^r \mathbb{E}\!\left[\sum_{t=0}^{\infty}\gamma^t \lVert\sigma_{\mathrm{epi}_n}^r(s_{t+H},a_{t+H})\rVert\right] \\
&\quad + \gamma^H \mathbb{E}\!\left[\sum_{t=0}^{\infty}\gamma^t \mu_n^r(s_{t+H},a_{t+H})\right] \\
&\le \lambda_n \tilde\Sigma_{\gamma,n}^d(\pi)
  + \beta^r_n \tilde\Sigma_{\gamma,n}^r(\pi,\mu_n^d)
  + \mathbb{E}\!\left[\sum_{t=0}^{H-1}\gamma^t \mu_n^r(s_t,a_t)\right] \\
&\quad + \gamma^H (c_n\mathbb{E}\left[\sigma_{\mathrm{epi}_n}^q(s_H,a_H)\right] +\mathbb{E}\!\left[\bar Q_{\mu_n}^\pi(s_H,a_H)\right]) \\
&\le \lambda_n \tilde\Sigma_{\gamma,n}^d(\pi)
  + \beta^r_n \tilde\Sigma_{\gamma,n}^r(\pi)
  + \mathbb{E}\!\left[\sum_{t=0}^{H-1}\gamma^t \mu_n^r(s_t,a_t)\right] \\
  &\quad + \gamma^H \mathbb{E}\!\left[\hat Q_{\mu_n}^\pi(s_H,a_H) + c_n \sigma_n^q(s_H,a_H)\right]
  + \gamma^H \epsilon_{q,n} \\
&= \eta_{\gamma,n}^{\mathrm{UBP}}(\pi,\mu_n^d) + \gamma^H \epsilon_{q,n},
\end{align*}}
where $c_n=\frac{\max\{\lambda_n,\beta_n^r\}}{\min\{\alpha_r,\alpha_d\}}\in\mathcal{O}\!\big(\max\{\sqrt{\Gamma_{r,N}},\sqrt{\Gamma_{d,N}}\}\big)$.

$\hfill\Box$

\begin{lemma} [Single episode regret bound for infinite-horizon return] \label{lem:single-regret-bound-inf}

Let Assumptions \ref{ass:cont-dyn-and-r}, \ref{ass:RKHS-regularity}, and \ref{ass:no-plan-subopt} hold, and assume that the learned dynamics and reward are well-calibrated. Consider the definition of $r_n$ in \eqref{eq:inf-regret}, and let $a_n,b_n,c_n$ be as defined in Lemma \ref{lem:ucb-optimism} . Then, $\forall n>0$ with probability at least $1-\delta$,
\begin{equation}
r_{\gamma,n}\le
        \big(b_n'^2 +b_n' c_n' + 2b_n'+c_n'\big)\,
        \Sigma_{\gamma,n}^d(\pi^{\mathrm{ind}}_n)
        + \big(2a_n^2 + a_n c_n' + 2a_n + c_n'\big)\,
        \Sigma_{\gamma,n}^r(\pi^{\mathrm{ind}}_n) + 2\gamma^H \epsilon_{q,n}\label{eq:single_regret_bound-inf}
\end{equation}
where $b'_n=\max\{b_n,\lambda_n\}$, $c_n'=c_n\max\{\alpha_r,\alpha_d\}$ and $\pi^{\mathrm{ind}}_n$ is the planner induced policy, assumed to maximize equation \eqref{eq:plan_objective}.
    
\end{lemma}

\begin{proof}
{\allowdisplaybreaks
\begin{align*}
r_n
&= \eta_\gamma(\pi^*,p^*) - \eta_\gamma(\pi_n,p^*)\\
&\le \eta^{\mathrm{UCB}}_{\gamma,n}(\pi^*,\mu^d_n)
    - \eta_\gamma(\pi^{\mathrm{ind}}_n,p^*)
    + \gamma^H \epsilon_{q,n}
 \qquad\;\;\;\;\;\;\;\;\;\;\;\text{(Lemma \ref{lem:inf-ucb-optimism})} \\
&\le \eta^{\mathrm{UCB}}_{\gamma,n}(\pi_n^*,\mu^d_n)
    - \eta_\gamma(\pi^{\mathrm{ind}}_n,p^*)
    + \gamma^H \epsilon_{q,n}
 \qquad\;\;\;\;\;\;\;\;\;\;\;\text{(Optimization of \eqref{eq:theory-plan-objective})} \\[2pt]
&= \eta^{\mathrm{UCB}}_{\gamma,n}(\pi^{\mathrm{ind}}_n,\mu^d_n)
   - \eta_\gamma(\pi^{\mathrm{ind}}_n,p^*)
   + \gamma^H \epsilon_{q,n}
 \qquad\;\;\;\;\;\;\;\;\text{(Assumption \ref{ass:no-plan-subopt})} \\
&= \eta(\pi^{\mathrm{ind}}_n,\mu_n)
  + a_n \tilde\Sigma_n^r(\pi^{\mathrm{ind}}_n)
  + b_n \tilde\Sigma_n^d(\pi^{\mathrm{ind}}_n)  \\
&\quad
  + \gamma^H
  \mathbb{E}_{\mu_n^d}\!\left[
    \hat Q_{\mu_n}^{\pi^{\mathrm{ind}}_n}(s_H,a_H) + c_n \sigma_{\mathrm{epi}_n}^q(s_H,a_H)
  \right]
  - \eta_\gamma(\pi^{\mathrm{ind}}_n,p^*)
  + \gamma^H \epsilon_{q,n}
 \\
&= \eta(\pi^{\mathrm{ind}}_n,\mu_n) - \eta_\gamma(\pi^{\mathrm{ind}}_n,p^*)
  + a_n \tilde\Sigma_n^r(\pi^{\mathrm{ind}}_n)
  + b_n \tilde\Sigma_n^d(\pi^{\mathrm{ind}}_n)  \\
&\quad
  + \gamma^H
  \mathbb{E}_{\mu_n^d}\!\left[
    \sum_{t=0}^\infty \gamma^H\big(\mu_n^r(s_{t+H},a_{t+H})-r(s_{t+H},a_{t+H})\big)
    + c_n \sigma_{\mathrm{epi}_n}^q(s_H,a_H)
  \right]
  + 2\gamma^H \epsilon_{q,n}
 \\
&\le (b_n'^2 + 2b_n')\,\Sigma_{\gamma,n}^d(\pi^{\mathrm{ind}}_n)
   + (a_n^2 + a_n)\,\Sigma_{\gamma,n}^r(\pi^{\mathrm{ind}}_n)
 \qquad \text{(Lemma B.7 from \cite{sukhija2025sombrlscalableoptimisticmodelbased})} \\
&\quad
  + \gamma^H
  \mathbb{E}_{\mu_n^d}\!\left[
    \beta_n^r\sum_{t=0}^\infty\gamma^t\sigma_{\mathrm{epi}_n}^r(s_t,a_t)
    + c_n' \sum_{t=0}^{\infty}\gamma^t
      \Big(\sigma_{\mathrm{epi}_n}^r(s_{t+H},a_{t+H})
      + \lVert \sigma_{\mathrm{epi}_n}^d(s_{t+H},a_{t+H}) \rVert \Big)
  \right]
  + 2\gamma^H \epsilon_{q,n}
 \\
&\le (b_n'^2 + 2b_n')\,\Sigma_{\gamma,n}^d(\pi^{\mathrm{ind}}_n)
   + (2a_n^2 + 2a_n)\,\Sigma_{\gamma,n}^r(\pi^{\mathrm{ind}}_n)
 \\
&\quad
  + \gamma^H
  \mathbb{E}_{\mu_n^d}\!\left[
    c_n' \sum_{t=0}^{\infty}\gamma^t
      \Big(\sigma_{\mathrm{epi}_n}^r(s_{t+H},a_{t+H})
      + \lVert \sigma_{\mathrm{epi}_n}^d(s_{t+H},a_{t+H}) \rVert \Big)
  \right]
  + 2\gamma^H \epsilon_{q,n}
 \\
&\le (b_n'^2 + b_n' c_n' + 2b_n' + c_n')\,\Sigma_{\gamma,n}^d(\pi^{\mathrm{ind}}_n)
   + (2a_n^2 + a_n c_n' + 2a_n + c_n')\,\Sigma_{\gamma,n}^r(\pi^{\mathrm{ind}}_n)
   + 2\gamma^H \epsilon_{q,n}
 .
\end{align*}}
where $c_n'=c_n\max\{\alpha_r,\alpha_d\}$. 

\end{proof}
This lemma yields our bound for the infinite-horizon regret. \\

\textit{Proof of Theorem \ref{theorem:inf-regret-bound}.} First note that since $a_n$, $b_n$ are monotonic, we know that $c_n'$ is also monotonic. 

{\allowdisplaybreaks
\begin{align*}
    R_N&=\sum_{n=1}^N r_n \\
      &\le\sum_{n=1}^N \big(b_n'^2 +b_n c_n' + 2b_n'+c_n'\big)\,
        \Sigma_{\gamma,n}^d(\pi^{\mathrm{ind}}_n)
        + \big(2a_n^2 + a_n c_n' + 2a_n + c_n'\big)\,
        \Sigma_{\gamma,n}^r(\pi^{\mathrm{ind}}_n) + 2\gamma^H \epsilon_{q,n} \\
     &\le \big(b_N'^2 +b_N c_N' + 2b_N'+c_N'\big) \sum_{n=1}^N  \Sigma_{\gamma,n}^d(\pi^{\mathrm{ind}}_n)+\big(2a_N^2 + a_N c_N' + 2a_N + c_N'\big)\sum_{n=1}^N\Sigma_{\gamma,n}^r(\pi^{\mathrm{ind}}_n) +2\gamma^H\sum_{n=1}^N \epsilon_{q,n}\\
     &:=B_N\sum_{n=1}^N  \Sigma_{\gamma,n}^d(\pi^{\mathrm{ind}}_n)+A_N\sum_{n=1}^N\Sigma_{\gamma,n}^r(\pi^{\mathrm{ind}}_n)+2\gamma^H\sum_{n=1}^N \epsilon_{q,n}.
\end{align*}}

Since $\mathcal{O}\!\big(\max\{\sqrt{\Gamma_{r,N}},\sqrt{\Gamma_{d,N}}\}\big)=\mathcal{O}(\sqrt{\Gamma_{r,N}}+\sqrt{\Gamma_{d,N}})$, we have $A_N,B_N\in\mathcal{O}(\Gamma_{r,N}+\Gamma_{d,N})$. Consider

{\allowdisplaybreaks
\begin{align*}
\sum_{n=1}^N \Sigma_{\gamma,n}^d(\pi^{\mathrm{ind}}_n)
&\le \sqrt{N}\,\sqrt{\sum_{n=1}^N \big(\Sigma_{\gamma,n}^d(\pi^{\mathrm{ind}}_n)\big)^2} \\
&\le \sqrt{N}\,\sqrt{\sum_{n=1}^N \mathbb{E}\!\left[\left(\sum_{t=0}^{\infty}\gamma^t \lVert \sigma_{\mathrm{epi}_n}^d \rVert\right)^2\right]} \\
&\le \sqrt{N}\,\sqrt{\sum_{n=1}^N \mathbb{E}\!\left[\left(\sum_{t=0}^{\infty}\gamma^t\right)\left(\sum_{t=0}^{\infty}\gamma^t \lVert \sigma_{\mathrm{epi}_n}^d(s_t,a_t) \rVert^2\right)\right]} \\
&= \sqrt{\frac{N}{1-\gamma}}\,\sqrt{\sum_{n=1}^N \mathbb{E}\!\left[\sum_{t=0}^{\infty}\gamma^t \lVert \sigma_{\mathrm{epi}_n}^d(s_t,a_t) \rVert^2\right]} \\
&\le \sqrt{\frac{C_\gamma N \Gamma_{N\log N}\log N}{1-\gamma} + \frac{C\,(\sigma^d_{\max})^2\,N\log N}{(1-\gamma)^2}}
\qquad \text{(Proof of Theorem.~5.5 from \cite{sukhija2025sombrlscalableoptimisticmodelbased})}
\end{align*}}
with $\sigma_{\max}$ as defined in Assumption \ref{ass:RKHS-regularity} and for some constants $C$ and $C_\gamma$. Following similar steps, we can prove an analogous inequality for $\Sigma_{\gamma,n}^r$. By defining $e_{q,N}=\sum_{n=1}^N \epsilon_{q,n}$, we obtain the bound in \eqref{eq:main-regret-inf}. 
$\hfill \Box$
\paragraph{Explicit infinite-horizon regret bound} While Theorem \ref{theorem:inf-regret-bound} is written in terms of the GP information gain, since both reward and dynamics are modeled as functions of state and action (i.e., input is \((s,a)\)), the input dimension in the kernel is \(d_s + d_a\). Therefore, once a kernel is chosen, the asymptotic growth of \(\Gamma_N\) explicitly depends on the dimension. We demonstrate that for each of the linear, RBF, and Matérn kernels, our regret bound remains sublinear.

Let $d_x = d_s + d_a$ and $M = N\log N$. Since both the reward model and each coordinate of the dynamics model are functions of the joint input $x = (s,a) \in \mathbb{R}^{d_x}$, our Theorem 5.8 gives:
$R_{\gamma,N}\leq H^3\sqrt N\Big(\Gamma_{d,M}^{3/2}+\Gamma_{r,M}^{3/2}\Big)+e_{q,N}$. Algebraically substituting the information-gain bounds for each of the kernels yields the following kernel-specific, dimension-dependent variants of Theorem 5.8:
\begin{itemize}
    \item $\textbf{Linear kernel}$:
    $\begin{aligned}[t]
    \Gamma_M
    &= \mathcal O(d_x\log M) \\
    &\quad\Longrightarrow\quad
    R_{\gamma,N}
    = \tilde{\mathcal O}\big(H^3\sqrt N d_x^{3/2}+e_{q,N}\big)
    \end{aligned}$

    \item $\textbf{RBF kernel}$:
    $\begin{aligned}[t]
    \Gamma_M
    &= \mathcal O(\log^{d_x+1} M) \\
    &\quad\Longrightarrow\quad
    R_{\gamma,N}
    = \tilde{\mathcal O}\big(H^3\sqrt N\log^{\frac32(d_x+1)}N+e_{q,N}\big)
    \end{aligned}$

    \item $\textbf{Matérn kernel}$:
    $\begin{aligned}[t]
    \Gamma_M
    &= \mathcal O\left(
    M^{\frac{d_x}{2\nu+d_x}}
    \log^{\frac{2\nu}{2\nu+d_x}}M
    \right) \\
    &\quad\Longrightarrow\quad
    R_{\gamma,N}
    = \tilde{\mathcal O}\left(
    H^3N^{\frac12+\frac{3d_x}{2(2\nu+d_x)}}+e_{q,N}
    \right)
    \end{aligned}$
\end{itemize}

Therefore, it is seen that for linear and RBF, that our regret bound remains sublinear in $N$ for fixed state and action dimension. For the Matérn with smoothness $\nu$, for sublinear growth a sufficient condition is $\frac{1}{2} + \frac{3d_x}{2(2\nu+d_x)} < 1$, or $d_x < \nu$. However, this may be too conservative. There exist work that show function optimization in the RKHS of a Matérn kernel obtain sublinear regret under broader conditions, such as \cite{janz2023banditoptimisationfunctionsmatern}.

\subsection{Analysis for Settings With Sub-Gaussian Process Noise}\label{sec:sub-gaussian-noise}
We now discuss how our infinite horizon regret bound extends beyond Gaussian process noise. In Assumption \ref{ass:cont-dyn-and-r}, we assume that the process noise of the true dynamics $p^*$ is Gaussian. As our theoretical work is current presented, the Gaussianity of the noise is necessary in two locations: (i) to obtain high-probability confidence intervals for the estimated dynamics and reward, and (ii) the invocation of Lemma B.7 from \cite{sukhija2025sombrlscalableoptimisticmodelbased}, which is a simulation-lemma type bound for the value function under the true and discounted dynamics. Therefore, relaxing the Gaussian noise to subgaussian noise requires the addressing of both these points.

The first use is not inherently Gaussian - we simply assumed so to simplify the analysis. The high-probability confidence relation required between the reward and dynamics model error and their posterior GP uncertainty can indeed be obtained under conditionally subgaussian noise. In particular, we refer to analysis done by \cite{chowdhury2017kernelizedmultiarmedbandits}; in the paper, the authors operate within an agnostic RKHS setting and also utilize the standard closed form for the GP posterior mean and variance as presented in \eqref{eq:gp-posterior}. It is explictly stated that the Gaussian prior assumption are only an aid to algorithm design, and that the true process noise may be a conditionally R-subgaussian martingale-difference sequence. They then provide a uniform confidence bound in their Theorem 2, stated below. Note that this bound is precisely the well-calibrated confidence bound our theory requires.
\begin{equation*}
    |\mu_{t-1}(x)-f(x)|
\le
\Big(B + R\sqrt{2(\gamma_{t-1}+1+\ln(1/\delta))}\Big)\sigma_{t-1}(x) \quad \text{(Theorem 2, [5])}
\end{equation*}\label{eq:theorem-2-chowdury}

However, as Lemma B.7 from \cite{sukhija2025sombrlscalableoptimisticmodelbased} relies on a Gaussian smoothing argument through their use of Lemma C.2 from \cite{kakade2020informationtheoreticregretbounds}, we need explicitly address settings with subgaussian noise. We will show that even under the subgaussian noise, \textit{the regret bound asymptotically remains sublinear}.

We assume that the true dynamics has additive subgaussian noise: $s_{t+1}=g^*(s_t,a_t)+\xi_t$, where $\xi_t$ is a conditionally mean-zero $\sigma_d$-subgaussian martingale-difference sequence. The learned dynamics is $\hat{s}_{t+1}=\mu_n^d (\hat{s}_t,a_t)+\xi_t$. To avoid overloading the notation, we denote the true dynamics as $g^*$. However, in actuality, $g^*=p^*$. We also make the following assumption:
\begin{assumption}[Lipschitzness of dynamics, policies, and reward] \label{ass:lipschitz-dyn-r-policy}
The true dynamics $g^*$, the reward $r$, and all $\pi \in \Pi$ are $L_g$, $L_r$, and uniformly $L_\pi$ Lipschitz respectively and satisfy a discounted stability condition $\gamma L_g(1+L_\pi)< 1$
\end{assumption}
An immediate corollary of the above assumption is the following:
\begin{corollary} [Lipschitz-ness of value function] \label{cor:v-lipschitz}
Let all conditions in Assumption \ref{ass:lipschitz-dyn-r-policy} hold. Then, the value function $V^\pi_g (s)$ induced by $g^*$ under any policy $\pi \in \Pi$ is uniformly $L_V$-Lipschitz with
\begin{equation}
    L_V \le \frac{L_r(1+L_\pi)}{1-\gamma L_g(1+L_\pi)}.
\end{equation}
\begin{proof}
Take any two starting states $s_0$ and $\tilde{s}_0$ and run the same deterministic policy from each state, yielding two actions, $a_0$ and $\tilde{a}_0$. Furthermore, consider the one-step dynamics which yield next states $s_1$ and $\tilde{s}_1$, where $s_1=g^*(s_0,a_0)+\xi_0$ and $\tilde{s}_1=g^*(\tilde{s}_0,\tilde{a}_0)+\xi_0$. 

Let $(s_0,a_0)$ denote the concatenated vector of $s_0$ and $a_0$, and let $(\tilde s_0,\tilde a_0)$ denote the concatenated vector of $\tilde{s}_0$ and $\tilde{a}_0$. We have that:
\begin{align*}
\|(s_0,a_0)-(\tilde s_0,\tilde a_0)\|
=&\sqrt{\|s_0-\tilde s_0\|^2-\|a_0-\tilde a_0\|^2}\le
\|s_0-\tilde s_0\|+\|a_0-\tilde a_0\| \\
\leq&\; (1+L_\pi)\|s_0-\tilde s_0\|.
\end{align*}
In the last step, we use the Lipschitz-ness of the policy, or the fact that $\|a_0-\tilde a_0\|
\le
L_\pi\|s_0-\tilde s_0\|.$
Now, we consider the difference between $s_1$ and $\tilde{s_1}$. Here, we cancel the process noise: $s_1-\tilde s_1
=
g(s_0,a_0)-g(\tilde s_0,\tilde a_0)$. Then, using the Lipschitz-ness of the dynamics we have the following:
\begin{align*}
\|s_1-\tilde{s}_1\| \leq L_g\|(s_0,a_0)-(\tilde s_0,\tilde a_0)\| \leq L_g(1+L_\pi)\|s_0-\tilde s_0\|=\rho\|s_0-\tilde s_0\|
\end{align*}
where we define $\rho = L_g(1+L_\pi)$

Again consider the one-step dynamics beginning at $s_1$ and $\tilde{s}_1$ after running the same deterministic policy from each state and taking actions $a_1$ and $\tilde{a}_1$. We obtain two more states $s_2=g^*(s_1, a_1)+\xi_1$ and $
\tilde{s}_2=g^*(\tilde{s}_1, \tilde{a}_1)+\xi_1$ It can easily be seen that $\|(s_1,a_1)-(\tilde s_1,\tilde a_1)\| \leq (1 + L_\pi)|s_1-\tilde s_1\|$ and hence $\|s_2-\tilde{s}_2\| \leq \rho^2\|s_0-\tilde s_0\|$ via the same argument as above. Iterating this argument, for some timestep $t$,
\begin{equation*}\label{eq:state-action-diff-bound}
    \|(s_t,a_t)-(\tilde s_t,\tilde a_t)\|
\le
(1+L_\pi)\|s_t-\tilde s_t\|.
\end{equation*}
\begin{equation*}\label{eq:state-diff-bound}
    \|s_t-\tilde{s}_t\| \leq \rho^t\|s_0-\tilde s_0\|
\end{equation*}

Now, we compare rewards. Using the Lipschitz-ness of $r$ and the equations above,
\begin{align*}
    |r(s_t,a_t)-r(\tilde s_t,\tilde a_t)|
&\le L_r\|(s_t,a_t)-(\tilde s_t,\tilde a_t)\| \notag \le
L_r(1+L_\pi)\|s_t-\tilde s_t\| \\
&\le
L_r(1+L_\pi)\rho^t\|s_0-\tilde s_0\|.
\end{align*}

We then define the value functions
\begin{equation*}
    V^\pi_g(s_0)
=
\mathbb E_\xi
\left[
\sum_{t=0}^{\infty}\gamma^t r(s_t,a_t)
\right]
\end{equation*}
\begin{equation}
    V^\pi_g(\tilde{s}_0)
=
\mathbb E_\xi
\left[
\sum_{t=0}^{\infty}\gamma^t r(\tilde{s}_t,\tilde{a}_t)
\right]
\end{equation}
Then, comparing one coupled rollout we have
\begin{align*}
    |V^\pi_g(s_0)-V^\pi_g(\tilde s_0)|
&\le
\sum_{t=0}^{\infty}
\gamma^t
|r(s_t,a_t)-r(\tilde s_t,\tilde a_t)| \\
&\leq 
\sum_{t=0}^{\infty}
\gamma^t
L_r(1+L_\pi)\rho^t
\|s_0-\tilde s_0\| \\
&\le
L_r(1+L_\pi)
\left(
\sum_{t=0}^{\infty}
(\gamma\rho)^t
\right)
\|s_0-\tilde s_0\|.
\end{align*}
Note that the infinte sum is only finite if $\gamma \rho = \gamma L_g(1+L_\pi)< 1$, which is precisely the discounted stability condition required in Assumption \ref{ass:lipschitz-dyn-r-policy}. If this is the case, then, the result of the infinte sum is $\sum_{t=0}^{\infty}
(\gamma\rho)^t
=
\frac{1}{1-\gamma\rho}.$
Therefore,
\begin{equation*}
    |V^\pi(s_0)-V^\pi(\tilde s_0)|
\le
\frac{L_r(1+L_\pi)}
{1-\gamma L_g(1+L_\pi)}
\|s_0-\tilde s_0\|.
\end{equation*}
So, for any policy $\pi \in \Pi$ and any starting states, $V^\pi_g$ is $L_V$-Lipschitz with $L_V$ upper bounded by the coefficient above.
\end{proof}
\end{corollary}

Now, we introduce the following lemma, which is a simulation lemma similar to Lemma \ref{lem:return_bounds} and Lemma B.7 of \cite{sukhija2025sombrlscalableoptimisticmodelbased}. It will be shown that 
\begin{lemma}[Simulation lemma for settings  with subgaussian noise] \label{lem:sim-lem-subgaussian}
Let Assumptions \ref{ass:cont-dyn-and-r}, \ref{ass:RKHS-regularity}, \ref{ass:lipschitz-dyn-r-policy} hold, and let $L_V$ be as defined in Corollary \ref{cor:v-lipschitz}. Then, there is a $\lambda_{n,\mathrm{SG}}^d = \gamma L_V \beta_n^d.$ such that for every policy $\pi \in \Pi$ with probability at least $1 - \delta$, 
\begin{equation}\label{eq:sim-lem-sg-true-dyn}
    \left|\eta_\gamma(\pi,p^*) - \eta_\gamma(\pi,\mu^d_n) \right| \le \lambda_{n,\mathrm{SG}}^d \Sigma_{\gamma,n}^d(\pi)
\end{equation}
\begin{equation}\label{eq:sim-lem-sg-learned-dyn}
\left|\eta_\gamma(\pi,p^*)-\eta_\gamma(\pi,\mu_n^d)\right|\le \lambda_{n,\mathrm{SG}}^d \widetilde\Sigma_{\gamma,n}^d(\pi),
\end{equation}
\begin{proof}
For simplicity, assume there is a deterministic policy $\pi$. The stochastic policy variant of the proof is the same, with an extra expectation over the action.

Define two value functions, $V_*(s_0)=V_{g^*}^{\pi}(s_0)$ and
$V_\mu(s_0)=V_{\mu^d_n}^{\pi}(s_0)$, corresponding to true dynamics and learned dynamics. Consider their Bellman expansions, which are $V_*(s_0)=r(s,\pi(s_0)) + \gamma\mathbb E_{\xi_0}[V_*(g^*(s_0,\pi(s_0))+\xi_0)]$ and $V_\mu(s_0)=r(s_0,\pi(s_0))+ \gamma \mathbb E_{\xi_0}[V_\mu(\mu_n^d(s_0,\pi(s_0))+\xi_0)].$ Let $\Delta(s)=V_\mu(s)-V_*(s)$. Then,
\begin{align*}
|\Delta(s_0)|
&= |
\gamma\mathbb E_{\xi_0}
\left[
V_\mu(\mu_n^d(s_0,\pi(s_0))+\xi_0)
\right]
-
\gamma\mathbb E_{\xi_0}
\left[
V_*(g^*(s_0,\pi(s_0))+\xi_0)
\right]| \\
&= |\gamma\mathbb E_{\xi_0}
\left[
V_\mu(\mu_n^d(s_0,\pi(s_0))+\xi_0)
-
V_*(\mu_n^d(s_0,\pi(s_0))+\xi_0)
\right] \\
&\qquad
+ \gamma\mathbb E_{\xi_0}
\left[
V_*(\mu_n^d(s_0,\pi(s_0))+\xi_0)
-
V_*(g^*(s_0,\pi(s_0))+\xi_0)
\right]|  \\
&\leq
\gamma
\mathbb E_{\xi_0}
[
|\Delta(s_1^\mu)|
]
+
\gamma
\mathbb E_{\xi_0}
[
|V_*(s_1^\mu)-V_*(s_1^*)|
] \\
&\leq 
\gamma
\mathbb E_{\xi_0}
[
|\Delta(s_1^\mu)|
]
+
\gamma
\mathbb E_{\xi_0}
[
L_V\|s_1^\mu - s^*_1\|
] \qquad \qquad \text{(Lipschitz-ness of $V_*$)}\\ 
&= \gamma
\mathbb E_{\xi_0}
[
|\Delta(s_1^\mu)|
]
+
\gamma
\mathbb E_{\xi_0}
[
L_V\|\mu_n^d(s_0, \pi(s_0)) - g^*(s_0, \pi(s_0))\|
] \\
&\leq
\mathbb E_{\xi_0}
[
|\Delta(s_1^\mu)|
]
+
\gamma
\mathbb E_{\xi_0}
[
L_V\beta^d_n\|\sigma^d_{\mathrm{epi}_n}(s_0, \pi(s_0))\|
] \qquad \text{(Definition \ref{def:well-calibrated})} \\
&=
\mathbb E_{\xi_0}
[
|\Delta(s_1^\mu)|
]
+
\gamma
L_V\beta^d_n\|\sigma^d_{\mathrm{epi}_n}(s_0, \pi(s_0))\|
\end{align*}
where $s_{t+1}^\mu=\mu_n^d(s_t,\pi(s_t))+\xi_t,$ and $s_{t+1}^*=g^*(s_t,\pi(s_t))+\xi_t.$ Now, we repeat the above process for the next state produced by the learned model $s^\mu_1$, yielding
\begin{equation*}
|\Delta(s_1^\mu)|
\le
\gamma
\mathbb E_{\xi_1}[
|\Delta(s_2^\mu)|
]
+
\gamma L_V\beta_n^d
\|\sigma_\mathrm{epi_n}^d(s_1^\mu,\pi(s_1^\mu))\|
\end{equation*}
Substituting into the previous inequality, we obtain
\begin{equation*}
|\Delta(s_0)|
\le
\gamma L_V\beta_n^d
\|\sigma_\mathrm{epi_n}^d(s_0,\pi(s_0))\| +
\gamma^2
L_VC_d\beta_n^d
\mathbb E_{\xi_0}[
\|\sigma_\mathrm{epi_n}^d(s_1^\mu,\pi(s_1^\mu))\|
]
+
\gamma^2
\mathbb E_{\xi_0}\mathbb E_{\xi_1}[
|\Delta(s_2^\mu)|
].
\end{equation*}
We repeat this expansion of $s^\mu_t$ forever. Notice that the last term vanishes as $t \to \infty$ due to the boundedness of the reward. Therefore, 
\begin{equation*}
|\Delta(s_0)|=|V_\mu(s_0)-V_*(s_0)|
\le
\lambda_{n,\mathrm{SG}}^d
\mathbb E_{\xi_0}
\left[
\sum_{t=0}^{\infty}
\gamma^t
\|\sigma_\mathrm{epi_n}^d(s_t,a_t)\|
\mid s_0
\right]
\end{equation*}
where $\lambda_{n,\mathrm{SG}}^d
=
\gamma L_V\beta_n^d$. Using \eqref{inf_sigma_sum_def3} and taking the expectation over the starting state $s_0$ yields \eqref{eq:sim-lem-sg-learned-dyn}.

We can reuse the above procedure but instead add and subtract $\gamma\mathbb E_{\xi_0}[V_\mu(g^*(s_0,\pi(s_0))+\xi_0]$ rather than $\gamma \mathbb E_{\xi_0}[V_*(\mu^d_n(s_0,\pi(s_0))+\xi_0]$. This will result in terms involving $\Delta(s^*_t)$ after repeating the same expansion step, yielding \eqref{eq:sim-lem-sg-true-dyn}.
\end{proof}
\end{lemma}
\begin{theorem}[Infinite-horizon regret bound in settings with subgaussian noise] \label{thm:inf-regret-subgaussian}
Let all conditions and assumptions stated in Theorem \ref{theorem:inf-regret-bound} hold, except that the process noise is a conditionally mean-zero $\sigma_d$-subgaussian martingale-difference sequence. Then, the cumulative infinite-horizon regret after $N$ episodes of real environment interaction satisfies, with probability $1 - \delta$, the same sublinear bound as presented in Theorem \ref{theorem:inf-regret-bound} up to constants/log factors:
\begin{equation*}
R_{\gamma,N}
\!\le\!
\mathcal{O}\;\Big(
H^3 \sqrt{N}\,
\big(\Gamma_{d,N\log N}^{3/2}\!+\!\Gamma_{r,N\log N}^{3/2}\big)\!+\!e_{q,N}
\Big)
\end{equation*}
where $e_{q,N}:=\sum_{n=1}^N\epsilon_{q,n}$.
\begin{proof}
We can see that Lemma \ref{lem:sim-lem-subgaussian} has the same form of Lemma B.7. in \cite{sukhija2025sombrlscalableoptimisticmodelbased}. The subgaussian replacement only changes the coefficient of the simulation lemma. In the Gaussian proof, the coefficient $\lambda_n
=
C_{\max}
\frac{\gamma}{1-\gamma}
\frac{(1+\sqrt{d_x})\beta_{n-1}}{\sigma}$, while in the subgaussian proof, we instead use $\lambda_{n,\mathrm{SG}}^d=\gamma L_V \beta_n^d.$ Since $\beta_n^d =
\mathcal O(\sqrt{\Gamma_{d,N}}),$ the coefficient has the same information-gain order as the original coefficient. Therefore, the regret rate remains the same up to constants/log-factors.
\end{proof}
\end{theorem}

\subsection{Analysis of UCB Objective With Planner Suboptimality} \label{sec:lemma-proofs-plan-sub}
\paragraph{Introducing Planner Sub-optimality.}
Recall Assumption~\ref{ass:no-plan-subopt}, where we assumed that the policy $\pi^{\mathrm{ind}}_n$ induced by the finite-horizon optimization of~\eqref{eq:theory-plan-objective} after the $n^{th}$ episode of real-environment interaction coincides with the ideal policy $\pi^*_n$ induced by exact maximization of~\eqref{eq:theory-plan-objective}. This assumption effectively treats the planner as an oracle that, under MPC-style re-planning, always selects the globally optimal first action according to the optimistic objective. In practice, however, planning is implemented via approximate optimization over some non-convex objective, hence the final solution action sequence may be suboptimal with respect to the planning objective \cite{sikchi2021learningoffpolicyonlineplanning, argenson2021modelbasedofflineplanning, jeong2025reflectthenplanofflinemodelbasedplanning}. Hence, it is prudent to also consider this imperfection in the theoretical analysis. 

We use the formulation introduced in \cite{sikchi2021learningoffpolicyonlineplanning}, and introduce a \emph{planner sub-optimality} parameter $\epsilon_p$, which upper bounds the gap between the optimistic objective value attained by the ideal MPC-induced policy and that achieved by the implemented planner. Using this, we re-state Assumption \ref{ass:no-plan-subopt}:

\begin{assumption} [Planner suboptimality in the optimistic objective]\label{ass:planner-subopt}
At the $n^{th}$ episode of real environment interaction, with each episode containing $T$ timesteps, there exist nonnegative numbers $\{\epsilon_{p,n,t}\}_{t=0}^{T-1}$ such that the policy induced by the UCB objective \eqref{eq:theory-plan-objective} satisfies:
\begin{equation}
\eta_{\gamma,n}^{\mathrm{UCB}}(\pi_n^{*},\mu_n^d)
-
\eta_{\gamma,n}^{\mathrm{UCB}}(\pi_n^{\mathrm{ind}},\mu_n^d)
\;\le\;
\epsilon_{p,n}^{\gamma},
\qquad
\epsilon_{p,n}^{\gamma}:=\sum_{t=0}^{T-1}\gamma^t\,\epsilon_{p,n,t}.
\label{eq:mpc_eps_def}
\end{equation}
In particular, if $\epsilon_{p,n,t}\le \epsilon_p$ for all $t$, then
$\epsilon_{p,n}^{\gamma}\le \frac{1-\gamma^T}{1-\gamma}\,\epsilon_p$.
\end{assumption}

Because MPC re-solves the UCB objective at each timestep and executes only the first action, suboptimality can accumulate across re-planning steps; we therefore quantify a discounted, episode-level planner error $\epsilon_{p,n}^{\gamma}$, where $\epsilon_{p,n,t}$ bounds the one-step optimization error at real environment step $t$ in episode $n$. This is unlike \cite{sikchi2021learningoffpolicyonlineplanning}, where the MPC setting is not considered. Ultimately, our definition yields a direct additive degradation in the per-episode regret bound, which we state below.

\begin{lemma}[Single episode regret with planner suboptimality] \label{lem:single-regret-plan-subopt}
Let Assumptions \ref{ass:cont-dyn-and-r}, \ref{ass:RKHS-regularity}, \ref{ass:q-suboptimality}, \ref{ass:q-uncertainty}, and \ref{ass:planner-subopt} hold (note: Assumption \ref{ass:no-plan-subopt} no longer holds), and assume that the learned dynamics and reward are well-calibrated. Consider the definition of $r_n$ in \eqref{eq:inf-regret}, and let $a_n, b_n$ be as defined in Lemma \ref{lem:ucb-optimism}, Then, $\forall n>0$ with probability at least $1-\delta$,

\begin{equation}
r_n\leq(b_n'^2 + b_n' c_n' + 2b_n' + c_n')\,\Sigma_{\gamma,n}^d(\pi^{\mathrm{ind}}_n)
   + (2a_n^2 + a_n c_n' + 2a_n + c_n')\,\Sigma_{\gamma,n}^r(\pi^{\mathrm{ind}}_n)
   + 2\gamma^H \epsilon_{q,n}
+
\epsilon_{p,n}^{\gamma},
\label{eq:single_regret_bound_mpc}
\end{equation}

where $a_n,b'_n,c'_n$ are as defined in Lemma \ref{lem:single-regret-bound-inf} and $\epsilon_{p,n}^{\gamma}$ is as defined in Assumption \ref{ass:planner-subopt}
\end{lemma}
\begin{proof} We follow the proof of Lemma \ref{lem:single-regret-bound-inf}:

\begin{align}
r_n
&=\eta_\gamma(\pi^*,p^*)-\eta_\gamma(\pi_n,p^*) \nonumber\\
&\le
\eta_{\gamma,n}^{\mathrm{UCB}}(\pi^*,\mu_n^d)-\eta_\gamma(\pi_n^{\mathrm{ind}},p^*)+\gamma^H\epsilon_{q,n} \nonumber\\
&= \nonumber
\underbrace{
\Big(\eta_{\gamma,n}^{\mathrm{UCB}}(\pi^*,\mu_n^d)-\eta_{\gamma,n}^{\mathrm{UCB}}(\pi_n^{*},\mu_n^d)\Big)
}_{(\mathrm{A})}
+
\underbrace{
\Big(\eta_{\gamma,n}^{\mathrm{UCB}}(\pi_n^{,*},\mu_n^d)-\eta_{\gamma,n}^{\mathrm{UCB}}(\pi_n^{\mathrm{ind}},\mu_n^d)\Big)
}_{(\mathrm{B})}
\label{eq:loop_decomp_1}\\
&\hspace{2.2em}
+
\underbrace{
\Big(\eta_{\gamma,n}^{\mathrm{UCB}}(\pi_n^{\mathrm{ind}},\mu_n^d)-\eta_\gamma(\pi_n^{\mathrm{ind}},p^*)\Big)+\gamma^H\epsilon_{q,n}
}_{(\mathrm{C})}.
\nonumber
\end{align}

Since $\pi_n^{*}\in\arg\max_{\pi} \eta_{\gamma,n}^{\mathrm{UCB}}(\pi,\mu_n^d)$ by definition, $(\mathrm{A})\leq0$. Notice that term (B) is precisely the definition of the MPC planner suboptimality in Assumption \ref{ass:planner-subopt}, so $(\mathrm{B}) \; \le \; \epsilon_{p,n}^{\gamma}$. Finally, notice that $(\mathrm{C})$ is equivalent to the original regret bound presented in the proof of Lemma \ref{lem:single-regret-bound-inf}. Putting the terms together, we obtain the regret bound in \eqref{eq:single_regret_bound_mpc}.
\end{proof}

\begin{theorem} [Regret bound with planner suboptimality] \label{theorem:regret_w_plan_subopt} Let Assumptions \ref{ass:cont-dyn-and-r}, \ref{ass:RKHS-regularity}, \ref{ass:q-suboptimality}, \ref{ass:q-uncertainty}, and \ref{ass:planner-subopt} hold (note: Assumption \ref{ass:no-plan-subopt} no longer holds), and assume that the learned dynamics and reward are well-calibrated. Let $\epsilon_{p,n}^\gamma$ be as defined in Assumption \ref{ass:planner-subopt}. Then, the cumulative regret after $N$ real environment episodes satisfies, with probability $1-\delta$:
\begin{equation}
R_{\gamma,N}
\;\le\;
\mathcal{O}\!\Big(
H^3 \sqrt{N}\,
\big(\Gamma_{d,N\log N}^{3/2} + \Gamma_{r,N\log N}^{3/2}\big)+e_{q,N}
\Big)+\sum_{n=1}^{N}\epsilon_{p,n}^\gamma,
\label{eq:main-regret-inf-plan-subopt}
\end{equation}
Moreover, if $\epsilon^\gamma_{p,n}\leq \epsilon^\gamma_{p,n,t}$ for all $n,t$, then:
\begin{equation}
R_{\gamma,N}
\;\le\;
\mathcal{O}\!\Big(
H^3 \sqrt{N}\,
\big(\Gamma_{d,N\log N}^{3/2} + \Gamma_{r,N\log N}^{3/2}\big)+e_{q,N}
\Big)
\;+\;
N\cdot \frac{1-\gamma^{T}}{1-\gamma}\,\epsilon_p
\label{eq:final_regret_const_eps}
\end{equation}
\end{theorem}
\begin{proof}
The single episode regret bound from Lemma \ref{lem:single-regret-plan-subopt} can be substituted into the proof procedure for Theorem \ref{theorem:inf-regret-bound} with minimal changes. The planner suboptimality simply introduces an additive term which has no effect on the rest of the proof.
\end{proof}

\paragraph{Planner suboptimality in the finite horizon regret bound} It is easily shown that the same additive terms corresponding to the planner suboptimality which appear in \eqref{eq:main-regret-inf-plan-subopt} and \eqref{eq:final_regret_const_eps} appear in the finite horizon regret bound in Theorem \ref{theorem:main-regret-bound}, contributing linearly. 

\paragraph{Discussion of planner suboptimality and regret} Evidently, in order for the cumulative regret to remain sublinear, the cumulative planner error should be sublinear rather than linear (as it currently is in Theorem \ref{theorem:regret_w_plan_subopt} . If an assumption is made that $\epsilon^\gamma_{p,n}\in\mathcal{O}(1/\sqrt{n})$, then $\sum_{n=1}^N \epsilon_{p,n}^\gamma \in O(\sqrt{N})$ and the overall regret remains sublinear.

The question then is whether such a decay assumption is warranted in practice. Under a fixed planning budget per replanning step (e.g., a constant number of trajectory samples or gradient steps), there is generally no reason to expect $\epsilon_{p,n,t}$ to vanish with $n$, in which case $\sum_{n=1}^N \epsilon_{p,n}^{\gamma}$ is linear and can dominate the regret bound asymptotically. On the other hand, a decaying $\epsilon^\gamma_{p,n}$ becomes plausible if planning accuracy improves over time, either by allocating increasing compute to the planner or by reducing the effective optimization difficulty through warm-starting and amortization. Past work has demonstrated that in nonconvex stochastic optimization, many algorithmic guarantees improve monotonically with the number of iterations or gradient evaluations, yielding smaller optimization residuals as compute increases \cite{ghadimi2013stochasticfirstzerothordermethods}. Similarly, sampling-based optimizers commonly used in MPC such as the cross-entropy method improve solution quality as the sampling budget increases~\cite{deboer2005ce}. Moreover, even without explicitly increasing the compute budget, which may not always be feasible, it has been shown that methodological innovations such as warm-starting and amortization has been explicitly studied as a mechanism to accelerate nonlinear trajectory optimization by predicting high-quality initial solutions \cite{byravan2021evaluatingmodelbasedplanningplanner, sambharya2022endtoendlearningwarmstartrealtime}. 

Taken together, these considerations suggest that in practice, conditions for assuming $\epsilon^\gamma_{p,n}\in\mathcal{O}(1/\sqrt{n})$ may be a reasonable abstraction in certain planning regimes with the aforementioned computational and methodological optimizations. In particular, practical TD-MPC2-style systems often warm-start planning \cite{hansen2024tdmpc}. However, establishing a general, specific decay rate for these methods remain challenging and are an interesting direction for future theoretical work.

\end{document}